\documentclass{article}

\usepackage{arxiv}

\usepackage{amssymb}

\usepackage{hyperref}       
\usepackage{url}            
\usepackage{booktabs}       
\usepackage{amsfonts}       
\usepackage{nicefrac}       
\usepackage{microtype}      
\usepackage{lipsum}
\usepackage{amsmath}
\usepackage{bm}
\usepackage{epsf} 
\usepackage{graphicx}
\usepackage{subfig}
\usepackage{booktabs,dcolumn,caption}
\usepackage{multirow}
\usepackage{float}
\usepackage{soul}
\usepackage{tikz}
\usepackage{physics}
\usepackage{mathdots}
\usepackage{yhmath}
\usepackage{cancel}
\usepackage{color}
\usepackage{array}
\usepackage{multirow}
\usepackage{amssymb}
\usepackage{gensymb}
\usepackage{tabularx}
\usepackage{amsthm}
\usepackage{xcolor}

\definecolor{darkgreen}{RGB}{0,100,0}

\usepackage{authblk}

\title{Breaking the Conventional Forward-Backward Tie in Neural Networks: Activation Functions}

\author[1,2]{Luigi Troiano}
\author[2,3]{Francesco Gissi}
\author[2,3]{Vincenzo Benedetto}
\author[1]{Genny Tortora}

\affil[1]{University of Salerno, Campus UniSA, 84084 Fisciano (SA), Italy}

\affil[2]{Kebula srl, Viale Filanda 3, 84080 Pellezzano (SA), Italy}

\affil[3]{University of Sannio, Viale Traiano, 82100 Benevento (BN), Italy}

\hypersetup{
  pdftitle={Breaking the Conventional Forward-Backward Tie in Neural Networks: Activation Functions},
  pdfauthor={Luigi Troiano, Francesco Gissi, Vincenzo Benedetto, Genny Tortora},
  pdfsubject={cs.LG, cs.NE},
  pdfkeywords={Back-propagation, Activation functions, Neural networks, Deep learning},
}
\begin{document}
\maketitle

\begingroup
\renewcommand\thefootnote{\fnsymbol{footnote}}
\footnotetext[1]{Corresponding author: \texttt{ltroiano@unisa.it}}
\endgroup

\begin{abstract}

{
Gradient-based neural network training traditionally enforces symmetry between forward and backward propagation, requiring activation functions to be differentiable (or sub-differentiable) and strictly monotonic in certain regions to prevent flat gradient areas. This symmetry, linking forward activations closely to backward gradients, significantly restricts the selection of activation functions, particularly excluding those with substantial flat or non-differentiable regions. In this paper, we challenge this assumption through mathematical analysis, demonstrating that precise gradient magnitudes derived from activation functions are largely redundant, provided the gradient direction is preserved. Empirical experiments conducted on foundational architectures—such as Multi-Layer Perceptrons (MLPs), Convolutional Neural Networks (CNNs), and Binary Neural Networks (BNNs)—confirm that relaxing forward-backward symmetry and substituting traditional gradients with simpler or stochastic alternatives does not impair learning and may even enhance training stability and efficiency. We explicitly demonstrate that neural networks with flat or non-differentiable activation functions, such as the Heaviside step function, can be effectively trained, thereby expanding design flexibility and computational efficiency. Further empirical validation with more complex architectures remains a valuable direction for future research.
}

\end{abstract}


\section{Introduction}

Neural network research has consistently prioritized error minimization, evaluating network predictions against reference outputs. Widrow and Lehr's comprehensive review \cite{Widrow90} systematically explores these learning methods, highlighting several fundamental weight updating rules summarized in Table \ref{tab:learning_rules}.

\begin{table}[t!]
\caption{A summary of the most relevant neural models with the associated weight updating rules, both in their scalar and vectorial form.}\label{tab:learning_rules}
\resizebox{\textwidth}{!}{
\begin{tabular}{@{}llll@{}}
\toprule
\textbf{Name} & $F(\bm{W};\bm{X})$ & \textbf{Update Rule (Scalar)} & \textbf{Update Rule (Vectorial)} \\ \midrule
  $\alpha$-LMS\\(Widrow-Hoff delta rule) 
        &  $\mathring{\bm{Y}} = \bm{W}\bm{X}$  
        &  $w + \alpha \frac{(\bm{y}-\bm{W}^T\bm{x})}{|\bm{x}|^2}\bm{x}$
        &  $\bm{W} + \frac{\alpha}{|\bm{X}|^2} (\bm{Y}-\bm{W}\bm{X})\bm{X}^T$ 
  \\
  \rule{0pt}{7ex}
  Rosenblatt $\alpha$-Perceptron 
        & $\mathring{\bm{Y}} = \bm{1}(\bm{W}\bm{X} - \bm{b})$  
        & $w + \alpha \frac{(\bm{y}-\bm{1}(\bm{W}^T\bm{x}-\bm{b}))}{2}\bm{x}$
        & $\bm{W} + \frac{\alpha}{2} \left(\bm{Y}-\bm{1}(\bm{W}\bm{X}-\bm{b})\right)\bm{X}^T$
  \\ 
  \rule{0pt}{7ex}
  Mays's Rule 1
        & $\mathring{\bm{Y}} = \bm{1}(\bm{W}\bm{X} - \bm{b})$  
        & $w +
                \begin{cases}
                    \alpha
                    \frac{ (\bm{y}-\bm{1}(\bm{W}^T\bm{x}-\bm{b}))}{2|\bm{x}|^2}\bm{x} & |\bm{W}^T\bm{x}| \geq \gamma  \\
                    \alpha
                    \frac{\bm{y}}{|\bm{x}|^2}\bm{x} & |\bm{W}^T\bm{x}| < \gamma 
                \end{cases}$
        & $\bm{W} +
                \begin{cases}
                    \frac{\alpha}{2|\bm{X}|^2} (\bm{Y}-\bm{1}(\bm{W}\bm{X}-\bm{b}))\bm{X}^T & |\bm{W}\bm{X}| \geq \gamma  \\
                    \frac{\alpha}{|\bm{X}|^2} \bm{Y}\bm{X}^T & |\bm{W}\bm{X}| < \gamma 
                \end{cases}$
  \\
  \rule{0pt}{7ex}
  Mays's Rule 2
        & $\mathring{\bm{Y}} = \bm{1}(\bm{W}\bm{X} - \bm{b})$  
        & $w +
                \begin{cases}
                    0 
                    & \mathring{\bm{y}} = \bm{y} \wedge |\bm{W}^T\bm{x}| \geq \gamma  \\
                    \alpha                \frac{(\bm{y}-\bm{W}^T\bm{x})}{|\bm{x}|^2}\bm{x} 
                    & otherwise
                \end{cases}$
        & $\bm{W} +
                \begin{cases}
                    0 
                    & \mathring{\bm{Y}} = \bm{Y} \wedge |\bm{W}\bm{X}| \geq \gamma  \\
                    \frac{\alpha}{|\bm{X}|^2} ( \bm{Y} - \bm{W}\bm{X})\bm{X}^T 
                    & otherwise
                \end{cases}$

  \\
  \rule{0pt}{7ex}
  $\mu$-LMS\\(Widrow-Hoff) 
        &  $\mathring{\bm{Y}} = \bm{W}\bm{X}$  
        &  $w - 2\mu(\bm{y}-\bm{W}^T\bm{x})\bm{x}$
        &  $\bm{W} - 2\mu (\bm{Y}-\bm{W}\bm{X})\bm{X}^T$ 
  \\
  \bottomrule
\end{tabular}}
\end{table}

It is essential to highlight the transformative impact that gradient descent has had on the development of neural network training methodologies. Gradient descent has not only streamlined the optimization process but also elegantly generalized the problem. Instead of viewing it as a mere iterative refinement, it reframed the challenge as a systematic quest to pinpoint the minimal loss within the vast and intricate parameter space \( \mathbf{W} \). This conceptual shift, interpreting optimization as a systematic journey through a high-dimensional landscape, has fundamentally shaped modern optimization techniques. Furthermore, it has enhanced our understanding of the intricate dynamics within neural network training landscapes, laying the groundwork for increasingly advanced and efficient training strategies, and setting the stage for increasingly sophisticated approaches.

Given a loss function \(\mathcal{L}\) as a function of the model's parameters \(\textbf{w}\) (i.e., the neural network's weights), the update rule at iteration \( t \) is
\begin{equation}\label{eq:deriv-update}
    \textbf{w}[t+1] = \textbf{w}[t] - \eta \frac{\partial \mathcal{L}(\textbf{w}[t])}{\partial \textbf{w}[t]}
\end{equation}
by moving against the loss gradient, guiding the parameters toward optimal solutions.

Here, \( t \) is the iteration index and \(\eta\) is the learning rate, dictating the step size in the gradient's opposing direction. Each iteration refines the weights by moving against the loss gradient, guiding the parameters toward optimal solutions, given the error hypersurface's convexity and continuity (i.e., differentiability or subdifferentiability). The iterative nature ensures progressive weight adjustments, reducing the loss and approaching the optimal set.

Historically, gradient descent has been a cornerstone optimization technique in computational research, with Cauchy's 1847 work \cite{cauchy1847} laying foundational groundwork. His iterative method in addressing numerical solutions became a foundational reference for many algorithms that followed. Fast-forwarding a century, Robbins and Monro \cite{robbins1951} in 1951 presented a stochastic variation of the gradient descent technique, effectively managing non-deterministic functions through stochastic approximation. This stochastic approach was particularly adept at handling non-deterministic functions, introducing an element of randomness to enable better convergence in diverse scenarios.

Entering the domain of artificial neural networks, Widrow and Hoff's \cite{widrow1960} Delta Rule stood out as one of the first instances where gradient-based methods were employed for learning. This rule essentially adjusted the weights of the network in the direction that would minimize the error, bridging the gap between traditional optimization techniques and neural network training.

The true paradigm shift in training multi-layer neural networks, however, came with the backpropagation algorithm, developed by Rumelhart, Hinton, and Williams \cite{rumelhart1986}. This algorithm, by efficiently computing gradients for multi-layer architectures, heralded the modern era of deep learning.

As architectures grew in depth and complexity, new optimization hurdles emerged. One such challenge was the pathological curvature, where specific regions in the optimization landscape made learning either painfully slow or unstable. This issue, deeply intertwined with the curvature of the loss landscape, brought to light another significant challenge: the gradient vanishing or explosion problem \cite{gradientexplosion}. The depth of networks, coupled with certain activation functions, could lead to gradients that either diminish to near-zero values or explode, obstructing stable and meaningful weight updates.

In the realm of optimizers, Stochastic Gradient Descent (SGD) \cite{sgd} stands out as a foundational algorithm. When enhanced with momentum calculations \cite{gradmomentum}, SGD provides a sophisticated approach to updating parameters based on their gradients during each training iteration. To address the inherent challenges of gradient-based optimization, the field has witnessed innovations on multiple fronts. This includes the advent of novel weight initialization strategies, the introduction of activation functions like ReLU that mitigate issues associated with small derivatives, and groundbreaking work on residual networks \cite{residualnets} that enable the training of even deeper architectures.

{
Although significant research attention has recently focused on diverse applications and architectures of neural networks (e.g., \cite{10534683,10628468,10870866}), foundational aspects of their training processes remain critically important and merit further investigation.

Several optimization techniques have been proposed to navigate these challenges. Notably, second-order approximations using Hessian-Free methods and the application of the inverse Fisher information matrix through Kronecker-factored approximations \cite{HFree, kroneckerfact} were introduced, and projection-type steepest descent neural networks specifically designed for nonsmooth optimization \cite{10.1016/j.neucom.2017.01.010} have been introduced.} The Adam optimizer by Kingma and Ba \cite{kingma2014adam} combined the strengths of AdaGrad \cite{duchi2011adaptive} and RMSProp \cite{tieleman2012lecture}. Other significant contributions include Nesterov's accelerated gradient approach \cite{nesterov1983method} and Zeiler's Adadelta \cite{zeiler2012adadelta}.

However, conventional gradient-based neural network training enforces a strict symmetry between forward propagation and backward propagation, inherently tying gradient computation directly to activation functions. This enforced symmetry significantly contributes to common training challenges such as vanishing or exploding gradients, neuron saturation \cite{neuron_saturation}, convergence to local minima \cite{local_minima}, and problematic curvature of the loss landscape \cite{HFree}. Although numerous strategies have been proposed to mitigate these issues \cite{gradcentralization, kingma2014adam, gradmomentum, weightnormalization}, these solutions generally remain constrained by conventional gradient symmetry.

In this work, we introduce a novel perspective by explicitly questioning the necessity of forward-backward symmetry in neural network training. Specifically, we address two fundamental questions:
\begin{itemize}
\item Is the standard construction of the back-propagation algorithm fundamentally required for effective neural network training?
\item Can we decouple the constraints linking forward activations and backward gradient computations?
\end{itemize}

{
The central aim of this manuscript is to critically evaluate, through mathematical analysis, the conventional symmetry assumption between forward and backward propagation in neural network training. Empirical experiments serve to validate these theoretical insights, rather than constituting standalone conclusions. In this paper, we focus on analyzing the theoretical feasibility of decoupling forward activation functions from backward gradient calculations. For this explicit purpose, simpler architectures (MLP, LeNet5) are intentionally selected to ensure clarity, mathematical rigor, and reproducibility. Although the theoretical results presented here hold in principle for a variety of neural network architectures, extending these insights practically and theoretically to more complex architectures, such as Transformers or graph neural networks, would require further theoretical development and empirical validation, explicitly identified as future work.
}

{
Our theoretical and experimental investigations reveal that the exact gradient magnitude derived from the activation function is not crucial for effective learning. Instead, our findings indicate that replacing the gradient magnitude obtained from the activation function with a magnitude derived from a different function, or even stochastic values, can yield comparable performance compared to traditional backpropagation. The objective of this study is not to critique or replace the gradient descent method but rather to explore and understand its fundamental properties within the context of neural network training:
\begin{enumerate}
    \item The strict coupling between the forward activation and backward gradient paths can be relaxed without impairing the network training.
    \item Although activation functions traditionally influence the magnitude of the gradient, the gradient's direction is solely determined by the linear connections between neurons, and plays a more critical role in the learning process.
    \item Decoupling forward and backward computations allows for innovative applications, such as effectively training binary neural networks with non-differentiable activation functions like the Heaviside step function.
\end{enumerate}

The symmetry-breaking approach presented here has several noteworthy consequences that provide novel insights into neural network training dynamics and may offer new possibilities for addressing longstanding training challenges:
\begin{itemize}
\item \textbf{Reduced Sensitivity to Gradient Magnitude}: Traditional backpropagation relies directly on activation function derivatives, often causing issues such as vanishing or exploding gradients. Altering the source of gradient magnitude can help mitigate these issues, potentially facilitating more stable convergence.

\item \textbf{Enhanced Stability and Flexibility}: By decoupling gradient magnitude computation, it becomes feasible to employ activation functions with substantial flat or non-differentiable regions, such as the Heaviside step function. This consequence increases flexibility in selecting activation functions and could improve robustness against neuron saturation.

\item \textbf{Prioritization of Gradient Direction over Magnitude}: Neural network training effectiveness significantly depends on gradient direction, largely determined by linear neuron connections. By defining gradient magnitude independently from activation derivatives, training processes may prioritize directional information, possibly streamlining and improving training efficiency.

\item \textbf{Improved Computational Efficiency}: Decoupling gradient magnitudes typically reduces computational complexity, potentially enabling faster convergence and decreased memory requirements.

\item \textbf{Expanded Scope for Novel Architectures and Applications}: Breaking forward-backward symmetry facilitates the training of architectures previously considered impractical, such as binary neural networks (BNNs) utilizing step activation functions, thereby extending the scope of feasible neural network designs. \end{itemize}

These outcomes highlight the potential implications of symmetry-breaking, encouraging further research into specialized gradient magnitude techniques tailored to specific architectures, investigation of neural network structures previously restricted by conventional training paradigms, and exploration of strategies aimed at enhancing the scalability and efficiency of neural networks for diverse applications, far beyond the scope of this work.
}

The remainder of this paper is structured as follows: Section 2 presents the theoretical foundations underpinning our study. In Section 3, we provide a rigorous mathematical analysis of forward-backward decoupling, explicitly demonstrating the feasibility of relaxing the traditional symmetry assumptions. Section 4 describes experimental evidence that validates these theoretical findings. Section 5 presents an illustrative application example, specifically focusing on Binary Neural Networks (BNNs). Finally, Section 6 offers concluding remarks and highlights promising avenues for further exploration.

{
\section{Theoretical Foundations}
}

Training neural networks, irrespective of their depth, presents challenges in its theoretical framework. The primary objective is to identify a parameter set that ensures reliable network inference, predicated on preliminary training using benchmark data. At the heart of the gradient descent method is the premise that networks are initially set up with arbitrary parameters. These parameters are successively adjusted over iterations \(t\), in the direction opposing the gradient of the target function. This iterative adjustment ensures that the loss function methodically approaches its optimal value. The learning rate, denoted by \(\eta\), dictates the magnitude of the adjustment during each iteration and, consequently, influences the iteration count required to achieve the optimal value. Building on the foundational principles of neural network training, a comprehensive review of the learning process grounded in back propagation is essential to understand the experimental results detailed in this paper.

For this study's context, we postulate the presence of an undisclosed binary relationship between two variables, \(\textbf{x}\) and \(\bm{y}\), expressed as \( \bm{y} \mathcal{R} \bm{x} \). These variables are encapsulated within appropriate representational domains. To simplify, we designate \( \bm{y} \in \mathbb{R}^p \) and \( \bm{x} \in \mathbb{R}^q \), with \(p, q \in \mathbb{N}\). The primary role of a neural network is to emulate the relation \(\mathcal{R}\) via a model \( \mathring{\bm{y}} = F(\bm{W}, \bm{x}) \), where the parameters are housed in \(\bm{W}\), representing the weights of the network's linear combinators. Given a dataset \( \mathcal{D} = [\bm{Y}|\bm{X}] \) - an empirical representation of the relation \( \bm{y} \mathcal{R} \bm{x} \) - our anticipation is for the network's output, \( \mathring{\bm{Y}} = F(\bm{W}; \bm{X}) \), to closely mirror the target \( \bm{y} \mathcal{R} \bm{x} \) for every \( \bm{x} \in \bm{X} \). Fundamentally, this encapsulates the aim of any training protocol, especially in the context of supervised learning.\footnote{Given this study's purview, our focus is confined to supervised learning.} Thus, the loss function \( \mathcal{L} \) we seek to refine should encapsulate the disparity between the network's output \( \mathring{\bm{Y}} \) and the target \( \bm{Y} \) across dataset \( \mathcal{D} \). Given that \( \mathring{\bm{Y}} \) is derived as a function \( F \) of the parameters \( \bm{W} \), it is logical to adopt \( \mathcal{L}(\bm{W}) \) without venturing into inconsistencies.

In light of the above, Eq.\eqref{eq:deriv-update} can be reformulated as:
\begin{equation}\label{eq:grad-update}
    \bm{W}[t+1] = \bm{W}[t] - \eta \nabla \mathcal{L}(\bm{W}[t])
\end{equation}

Given that the network is trained on a data subset — essentially a finite segment of the actual data — the iterative method denoted by Eq.\eqref{eq:grad-update} is more appropriately referred to as \emph{stochastic gradient descent} (SGD). Under this scenario, the loss function is proximated by an estimate achieved by evaluating the gradient in relation to the \(i\)-th observation and subsequently averaging it across numerous observations:
\begin{equation}
   \hat{\mathcal{L}}(\bm{W}) = E[\mathcal{L}(\bm{W};\bm{X})] = \frac{1}{|\bm{X}|} \sum_{\bm{x}\in \bm{X}} \mathcal{L}(\bm{W};\bm{X})
\end{equation}
where \( |\bm{X}| \) refers to the cardinality, i.e., the number of items, of \(\bm{X}\). Consequently, we can express:
\begin{equation}
    \bm{W}[t+1] = \bm{W}[t] - \eta \nabla \hat{\mathcal{L}}(\bm{W}[t])\ = \bm{W}[t] - \eta\nabla \mathcal{L}(\bm{W}[t])|_{\bm{X}},
\end{equation}
Here, the term \(\nabla \mathcal{L}(\bm{W}[t])|_{\bm{X}}\) is delineated as:
\begin{equation}
    \left. \nabla \mathcal{L}(\bm{W}[t]) \right|_{\bm{X}} = \frac{1}{|\bm{X}|} \sum_{\bm{x}\in \bm{X}} \nabla \mathcal{L}(\bm{W}[t];\bm{x}).    
\end{equation}
The variable \(t\) quantifies the function evaluations across the entire dataset \(\mathcal{D}\), referred to as \emph{epochs}. 

Partitioning the dataset into segments for separate processing each epoch is commonly referred to as \emph{batching}. When the batch size is \(|\mathcal{D}|\), iteration and epoch counts align. Typically, the dataset is divided into \(d\) segments, \([\bm{Y}_1|\bm{X}_1], \ldots, [\bm{Y}_d|\bm{X}_d]\), often of size \(m\) such that \(m \times d = |\mathcal{D}|\). This facilitates systematic SGD application on each segment, termed \emph{mini-batching}. Consequently, Eq.\eqref{eq:grad-update} becomes:
\begin{equation}\label{eq:minibatch-updata}
    \bm{W}[t,k+1] = \bm{W}[t,k] - \eta \nabla \hat{\mathcal{L}}(\bm{W}[t,k]) = \bm{W}[t,k] - \eta \left. \nabla \mathcal{L}(\bm{W}[t]) \right|_{\bm{X}} 
\end{equation}
With \(k\) spanning 1 to \(d\), note that at \(k=d\), we increment to the next epoch, i.e., \(t \rightarrow t+1\) and reset \(k=1\). In the extreme mini-batching case (\(m=1\)), parameter updates are sample-wise, often referred to as the \emph{stochastic mode}.

With the parameter update mechanism now delineated—be it for epoch or mini-batch, contingent on the present step's weights—we can streamline the notation for clarity, sidestepping the overt reliance on variables \(t\) and \(k\) in Eq.\eqref{eq:grad-update}. Hence, we represent:
\begin{equation}\label{eq:update-notime}
    \bm{W} = \bm{W} - \eta \left. \nabla_{\bm{W}} \mathcal{L} \right|_{\bm{X}}.    
\end{equation}

In today's era of deep learning, envisioning neural networks as amalgamations of multiple building blocks has become the norm. While various methodologies exist to arrange these blocks, an overwhelmingly prevalent and simplistic design incorporates a series of identical, repeating blocks, as portrayed in Fig.\ref{fig:neural-blocks}. For clarity and simplicity, our discussion will concentrate on this multi-layer configuration. Nevertheless, the insights and conclusions derived here can seamlessly be extrapolated to alternative network designs.

\begin{figure}[!ht]
\centering
\resizebox{\textwidth}{!}{
\includegraphics[page=1,width=\textwidth]{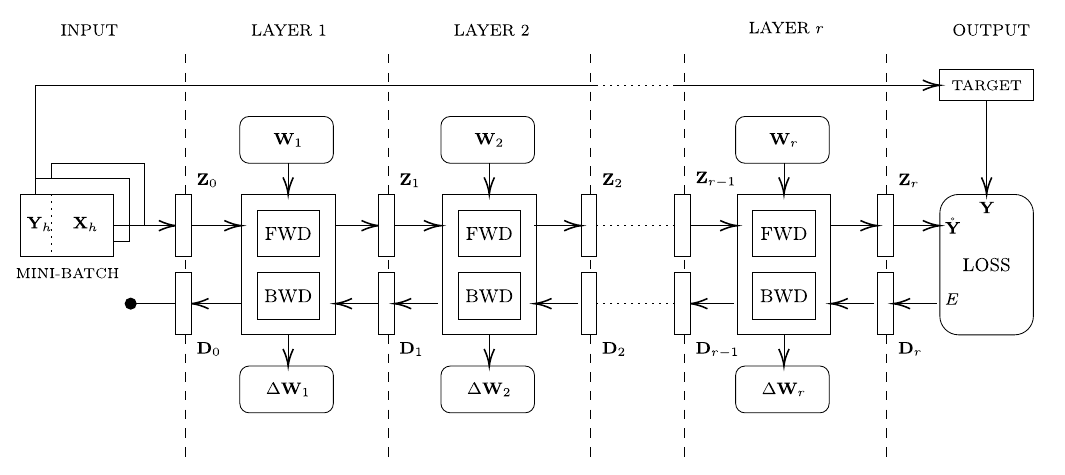}
}
\caption{Block diagram illustrating the training process of a multi-layer neural network made up of blocks. The arrows, progressing from left to right, signify the network's forward pass during training. Conversely, the backward pass retraces in the opposite direction.}\label{fig:neural-blocks} 
\end{figure}

The figure underscores the bi-directional (forward and backward) process during a single training iteration. The mini-batch \( [\bm{Y}_h|\bm{X}_h] \) serves as the network's input, where \( \bm{Y}_h \) signifies the anticipated network output, and \( \bm{X}_h \) is channeled into the first layer via an encoding variable, \( \bm{Z}_0 \). Specifically, for a network with \( r \) layers (excluding the input layer), the output for each layer is determined, T
throughout the forward propagation, by:
\begin{equation}
    \bm{Z}_l = G_l(\bm{W}_l,\bm{Z}_{l-1}) \quad \text{for} \quad l=1\dots r  
\end{equation}
In this equation, \( l \) refers to the current layer, ranging from the first (1) up to the \( r^{th} \) layer. Hence, the value of \( r \) explicitly signifies the total number of layers in the neural network, excluding the initial input layer. Hence, the correction of the weights at layer \( l \) is given by
\begin{equation}
    \Delta \bm{W}_l = -\eta \left. \nabla_{\bm{W}_l} \mathcal{L} \right|_{\bm{X}}.
\end{equation}

Here, \( G_l \) symbolizes the input/output relationship facilitated by the \( l \)-th layer, parametrically contingent on \( \bm{W}_l \). The output from every layer is retained in \( \bm{Z}_l \) due to its utility during backward propagation. The network's final layer's output is perceived as \( \mathring{\bm{Y}} \), serving as the loss function's input. This function then computes a deviation value, \( E=\nabla_{\mathring{\bm{Y}}} \mathcal{L} \), contrasting the forward computation's outcome with the target \( \bm{Y} \). This deviation then initiates the backward pass, denoted as \( \bm{D}_r = E \). For every block, potential parameter alterations are deduced based on the preceding block's output:
\begin{equation}\label{eq:W_update}
\Delta \bm{W}_l = - \eta \bm{D}_{l} \odot \nabla_{\bm{W}_l} G_l(\bm{W}_l,\bm{Z}_{l-1}),
\end{equation}
where \( \nabla_{\bm{W}_l} G_l(\bm{W}_l,\bm{Z}_{l-1}) \) is the Jacobian matrix of the function \( G_l \) with respect to the weights \( \bm{W}_l \) at layer \( l \) and has dimensions \( n_l \times n_{l-1} \). The Hadamard product \( \odot \) means that each element of \( \bm{D}_{l} \), having dimensions \( n_l \times 1 \), multiplies a corresponding row in the Jacobian matrix. Thus, \( \Delta \bm{W}_l \) has dimensions \( n_l \times n_{l-1} \), as expected.

Subsequently, the error is propagated back by
\begin{equation}\label{eq:D_update}
\bm{D}_{l-1} = \bm{D}_{l} \nabla_{\bm{Z}_{l-1}} G_l(\bm{W}_l,\bm{Z}_{l-1})
\end{equation}
where:
\begin{itemize}
\item[--] \(\bm{D}_{l-1}\) represents the error term for layer \(l-1\). Its dimension corresponds to the number of neurons in that layer, \(n_{l-1}\), making it a column vector with dimensions \([n_{l-1} \times 1]\).

\item[--] The term \(\nabla_{\bm{Z}_{l-1}} G_l(\bm{W}_l,\bm{Z}_{l-1})\) is the Jacobian matrix representing the gradient of function \(G_l\) with respect to the activations of layer \(l-1\). Given that layer \(l-1\) has \(n_{l-1}\) neurons and layer \(l\) has \(n_l\) neurons, this Jacobian matrix takes the dimension \([n_l \times n_{l-1}]\).

\item[--] \(\bm{D}_{l}\) is the error term for layer \(l\). With \(n_l\) neurons in layer \(l\), its dimensions are \([n_l \times 1]\).
\end{itemize}
Combining these elements, the multiplication of the Jacobian matrix (of size \([n_l \times n_{l-1}]\)) with the error term \(\bm{D}_{l}\) (of size \([n_l \times 1]\)) results in a column vector of dimensions \([n_{l-1} \times 1]\). This matches the expected size for the error term \(\bm{D}_{l-1}\). The backward propagation cycle concludes upon reaching the terminal block. 

The gradient descent method encounters challenges when the error hypersurface lacks continuity or convexity, resulting in potential issues such as becoming trapped in local minima, oscillating across valleys, stagnating on plateaus, and deviating from optimal regions.Several strategies have been introduced to mitigate these issues, including adaptive learning rates, incorporation of momentum, weight normalization, and modifications to activation functions.

Examining Eqq. \eqref{eq:W_update} and \eqref{eq:D_update}, we can discern that the performance and convergence of SGD are influenced by three fundamental elements: \emph{activations}, denoted as \( \bm{Z}_{l-1} \), \emph{weights}, represented by \( \bm{W}_l \), and \emph{gradients} of both activations and weights, symbolized as \( \nabla_{\bm{Z}_{l-1}} \) and \( \nabla_{\bm{W}_l} \) respectively. Thus, in delving into the complexities of SGD within the context of deep neural networks, the emphasis has largely been on tailoring these elements to enhance the efficacy of SGD. These interventions are briefly presented below.

\subsection{Activations}

Normalization methods play a critical role in regulating activations \( \bm{Z}_{l-1} \) as they propagate through neural network layers, significantly enhancing the efficacy of training. Various normalization strategies have been developed to manage challenges arising from the internal variability of activations. One influential approach involves controlling the Lipschitz constant of the loss function, thus promoting a smoother optimization landscape. The Lipschitz constant quantifies how sensitively a function responds to input variations; reducing it increases regularity, simplifies optimization, accelerates convergence, and diminishes the risk of becoming trapped in poor local minima.

Normalization methods generally follow the mathematical form:
\begin{equation}
\label{eq:activation_normalization}
    \tilde{\bm{Z}} = \gamma \cdot \frac{\bm{Z}-\mu (\bm{Z})}{\sqrt{\sigma^2(\bm{Z}) + \epsilon} } +\beta,
\end{equation}
where \( \mu(\bm{Z}) \) and \( \sigma^2(\bm{Z}) \) represent the mean and variance of activations computed along specified axes. The small constant \(\epsilon\) ensures numerical stability by preventing division by zero. Parameters \(\gamma\) (scale) and \(\beta\) (shift) are typically learnable, providing adaptive flexibility during training.

Below, we summarize several prominent normalization techniques.

\vspace{3mm}
\textbf{Batch Normalization (BN)}~\cite{ioffe2015batch} standardizes activations using statistics computed from each mini-batch, reducing internal covariate shift and accelerating convergence. Specifically, BN computes the mean and variance (\(\mu(\bm{Z})\), \(\sigma^2(\bm{Z})\)) across the mini-batch dimension. Parameters \(\gamma\) and \(\beta\) are learnable, typically initialized to \(\gamma=1\), \(\beta=0\). BN stabilizes activation distributions, smoothing the optimization landscape~\cite{bnSmooth}. However, its reliance on batch statistics makes it less effective for small batches common in certain computer vision tasks.

\vspace{3mm}
\textbf{Instance Normalization (IN)}~\cite{instancenormalization}, similar to BN, differs primarily by computing normalization statistics individually per sample across its feature dimensions. Formally, IN calculates \(\mu(\bm{Z})\) and \(\sigma^2(\bm{Z})\) independently per instance. An important variant, \textit{Adaptive Instance Normalization (AdaIN)}~\cite{adaIN}, adapts parameters to reflect style-specific statistics. Given a style tensor \(\overline{\bm{Z}}\), AdaIN defines \(\gamma = \sigma(\overline{\bm{Z}})\), \(\beta = \mu(\overline{\bm{Z}})\), aligning activations with stylistic attributes, and proving particularly effective in style-transfer applications.

\vspace{3mm}
\textbf{Layer Normalization (LN)}~\cite{layernormalization}, an alternative to BN, addresses scenarios sensitive to varying or small batch sizes. Unlike BN, LN normalizes independently across feature dimensions within each individual sample. Specifically, LN computes \(\mu(\bm{Z})\), \(\sigma^2(\bm{Z})\) across the feature dimensions of each sample. This characteristic makes LN advantageous in sequential data processing such as recurrent neural networks (RNNs) and tasks involving variable-length sequences common in natural language processing (NLP).

\vspace{3mm}
\textbf{Group Normalization (GN)}~\cite{groupnormalization} strikes a balance between instance-level and layer-level normalization, ensuring robust performance irrespective of batch size. GN partitions channels into \(G\) groups, normalizing each independently, computing statistics \(\mu_g(\bm{Z})\) and \(\sigma^2_g(\bm{Z})\) over spatial and grouped channel dimensions. This intermediate grouping mitigates limitations associated with small batch sizes and provides stability, making GN suitable for a wide range of applications.

\subsection{Weights}

Several re-parameterization techniques have emerged to optimize neural network weights effectively---fundamental for stable and efficient training---including Weight Normalization (WN), Weight Standardization (WS), and Weight Centralization (WC). Each method targets the weight distribution, facilitating smoother optimization landscapes, faster convergence, and enhanced stability.

\vspace{3mm}

\textbf{Weight Normalization (WN)}~\cite{weightnormalization} re-parameterizes weights by explicitly decoupling their magnitude and direction. Given an original weight tensor \(\bm{W}\), WN produces a normalized weight tensor \(\tilde{\bm{W}}\) through:
\begin{equation}
    \tilde{\bm{W}} = g \cdot \frac{\bm{W}}{\|\bm{W}\|},
\end{equation}
where \(\|\bm{W}\|\) denotes the Euclidean norm (magnitude) of the tensor, and \(g\) is a scalar parameter controlling its scale. By separating the scale from the direction explicitly, WN encourages smoother optimization dynamics, potentially accelerating training convergence and enhancing model stability.

\vspace{3mm}

\textbf{Weight Standardization (WS)}~\cite{weightstandardization} applies normalization directly to weights within layers, standardizing them to a consistent scale. Specifically, WS transforms each weight tensor \(\bm{W}\) as follows:
\begin{equation}
    \tilde{\bm{W}} = \frac{\bm{W}-\mu(\bm{W})}{\sigma(\bm{W})+\epsilon},
\end{equation}
where \(\mu(\bm{W})\) and \(\sigma(\bm{W})\) represent the mean and standard deviation of the weight tensor, respectively, and \(\epsilon\) is a small constant for numerical stability. WS smooths the optimization landscape by controlling the Lipschitz constant of the loss function, similarly to Batch Normalization, but acting directly on network weights rather than activations.

\vspace{3mm}

\textbf{Weight Centralization (WC)}~\cite{weightcentralization} explicitly enforces a zero-mean constraint on weight tensors. Given the original weight tensor \(\bm{W}\), WC computes a centralized tensor \(\tilde{\bm{W}}\) through:
\begin{equation}
    \tilde{\bm{W}} = \bm{W}-\mu(\bm{W}),
\end{equation}
where \(\mu(\bm{W})\) denotes the mean of the tensor \(\bm{W}\). By enforcing a zero-mean distribution, WC stabilizes weight updates, improves training dynamics, and often accelerates convergence. Importantly, WC can be generalized further by adjusting weights to specific mean or norm constraints, offering enhanced flexibility for diverse network architectures and tasks.

\subsection{Gradients}

Various gradient manipulation techniques have been developed to address challenges such as vanishing or exploding gradients, enhance stability, accelerate convergence, improve generalization, and facilitate escaping from poor local minima. Indeed, effectively managing gradients (\(\nabla_{\bm{Z}_{l-1}}\) and \(\nabla_{\bm{W}_l}\), collectively denoted as \(\nabla\)) is essential for the efficient training of neural networks, particularly within the framework of Stochastic Gradient Descent (SGD).

\vspace{3mm}
\textbf{Gradient Momentum}~\cite{gradmomentum, sutskever2013importance} mitigates oscillations and accelerates convergence by combining the current gradient with a fraction of the previous gradient update. Formally, momentum updates gradients as follows:
\begin{equation}
    \tilde{\nabla}[t] = \beta \cdot \tilde{\nabla}[t-1] + (1-\beta) \cdot \nabla \mathcal{L},
\end{equation}
where \(\tilde{\nabla}[t]\) is the momentum-adjusted gradient at time step \(t\), \(\beta\) is typically close to 1 (e.g., 0.9), and \(\nabla \mathcal{L}\) denotes the loss gradient. By accumulating velocity in consistent directions, momentum significantly enhances optimization dynamics, particularly in high-dimensional spaces.

\vspace{3mm}
\textbf{Gradient Clipping}~\cite{pascanu2013difficulty} directly addresses the exploding gradient issue by limiting gradient magnitude. Gradients exceeding a certain threshold \(\theta\) are scaled back proportionally:
\begin{equation}
    \tilde{\nabla} = \frac{\nabla}{\max\left(1, \frac{\lVert \nabla \rVert}{\theta}\right)},
\end{equation}
where \(\lVert \nabla \rVert\) represents the gradient norm. Properly chosen clipping thresholds stabilize training without excessively constraining beneficial updates, especially crucial in recurrent neural network training.

\vspace{3mm}
\textbf{Gradient Scaling}~\cite{nvidia_mixed_precision} counters vanishing gradients by explicitly scaling gradient magnitudes:
\begin{equation}
    \tilde{\nabla} = s \cdot \nabla,
\end{equation}
where \(s\) is a scalar factor. Appropriate gradient scaling ensures meaningful gradient updates, frequently employed together with gradient clipping to balance stability and convergence dynamics.

\vspace{3mm}
\textbf{Gradient Noise Addition}~\cite{neelakantan2015adding} introduces controlled randomness into gradient computations, improving robustness and generalization by reducing overfitting. The gradient with added noise is defined as:
\begin{equation}
    \tilde{\nabla} = \nabla + \zeta,
\end{equation}
where \(\zeta\) is typically Gaussian-distributed noise. Carefully tuned noise intensity balances exploration of the optimization landscape with training stability.

\vspace{3mm}
\textbf{Gradient Dropout} introduces stochastic regularization to gradient computations, analogous to dropout in activations. It randomly sets gradient components to zero:
\begin{equation}
    \tilde{\nabla} = \nabla \odot \mathbf{M},
\end{equation}
where \(\mathbf{M}\) is a binary mask sampled from a Bernoulli distribution with retention probability \(p\). Gradient dropout reduces reliance on specific paths or neurons, enhancing generalization and robustness.

\vspace{3mm}
\textbf{Gradient Centralization (GC)}~\cite{gradcentralization} explicitly imposes a zero-mean constraint on gradients, stabilizing training dynamics. For a gradient matrix \(\nabla\), centralized gradients \(\tilde{\nabla}\) are computed as:
\begin{equation}
    \tilde{\nabla}_{ij} = \nabla_{ij} - \frac{1}{c}\sum_{j=1}^{c}\nabla_{ij},
\end{equation}
where \(c\) represents the input dimension for each neuron or filter. GC mitigates neuron correlations, improves generalization, and reduces overfitting.

Empirical evidence consistently supports the efficacy of these gradient manipulation techniques. Importantly, all these methods influence gradient magnitudes directly or indirectly during backpropagation. Motivated by this unified perspective, in the subsequent section, we propose explicitly decoupling gradient adjustments from forward activations, rigorously demonstrating that gradient magnitude adjustments need not remain strictly coupled to forward computations.

{
\section{Mathematical Analysis of Forward-Backward Decoupling}

We now present the core theoretical propositions underpinning our analysis, accompanied by their detailed proofs.


\theoremstyle{definition}
\newtheorem{proposition}{Proposition}

\begin{proposition}[Gradient Direction Dominance]
The direction of weight updates during neural network training is independent of the activation function derivative and is predominantly determined by linear neuron connections.
\end{proposition}

\begin{proof}
Recall the weight update rule previously introduced in Eq.~\eqref{eq:W_update}:
\begin{equation*}
\Delta \bm{W}_l = - \eta \bm{D}_{l} \odot \nabla_{\bm{W}_l} G_l(\bm{W}_l,\bm{Z}_{l-1}),
\end{equation*}
and the backward gradient propagation equation from Eq.~\eqref{eq:D_update}:
\begin{equation*}
\bm{D}_{l-1} = \bm{D}_{l} \nabla_{\bm{Z}_{l-1}} G_l(\bm{W}_l,\bm{Z}_{l-1}).
\end{equation*}

Expanding explicitly for typical neural network configurations:
\begin{equation}
\nabla_{\bm{Z}_{l-1}} G_l(\bm{W}_l,\bm{Z}_{l-1}) = \bm{W}_l^\top \odot \varphi'_l(\bm{W}_l \bm{Z}_{l-1}),
\end{equation}
where \(\varphi'_l(\cdot)\) denotes the element-wise activation derivative at layer \(l\).

This expression reveals clearly that the directionality of gradients is dictated entirely by the linear transformation \(\bm{W}_l^\top\). The term \(\varphi'_l(\bm{W}_l \bm{Z}_{l-1})\), representing the derivative of the activation function, serves solely as an element-wise scalar multiplier. Hence, the derivative of the activation function does not alter the fundamental gradient direction, only its magnitude.

Since the directionality of gradient updates relies exclusively on the linear transformation defined by neuron connections \(\bm{W}_l^\top\), it underscores the dominant role of these linear components in determining the gradient update directions.

We conclude that gradient direction independence from activation derivatives and linear component dominance are fundamentally interconnected aspects of neural network training dynamics. Thus, gradient direction is dominated by linear neuron connections rather than activation function derivatives.
\end{proof}

\begin{proposition}[Generality of Gradient Direction Dominance]
Gradient direction dominance applies generally to neural network architectures, including CNNs, RNNs, transformers, and graph-based architectures.
\end{proposition}

\begin{proof}
Consider the backward gradient propagation at a generic network layer \( l-1 \), as previously defined in Eq.~\eqref{eq:D_update}:
\begin{equation*}
\bm{D}_{l-1} = \bm{D}_{l}\nabla_{\bm{Z}_{l-1}} G_l(\bm{W}_l,\bm{Z}_{l-1}),
\end{equation*}
where \(\bm{D}_{l}\) is the propagated gradient at layer \( l \), and \(G_l\) is the generic layer transformation.

For general neural network architectures—such as CNNs, RNNs, transformers, and graph-based networks—the gradient transformation \(\nabla_{\bm{Z}_{l-1}} G_l(\bm{W}_l,\bm{Z}_{l-1})\) can be represented as a linear transformation (via Jacobian matrices or linear operators specific to the given architecture)\footnote{The gradient propagation step involves the Jacobian or another linear operator specific to the network architecture (e.g., weight matrix transpose for fully connected layers, convolution transpose for CNNs, recurrence matrices for RNNs, attention-weighted combinations for Transformers, and linear aggregations for GNNs.
}, followed by an element-wise multiplication with the activation derivatives. Formally, this can be expressed as:
\begin{equation}
\nabla_{\bm{Z}_{l-1}} G_l(\bm{W}_l,\bm{Z}_{l-1}) = \mathbf{J}_l^\top \odot \varphi'_l(\bm{Z}_{l-1}),
\end{equation}
where \(\mathbf{J}_l\) denotes the linear Jacobian or equivalent linear operator defining the gradient propagation specific to the architecture under consideration, and \(\varphi'_l(\bm{Z}_{l-1})\) is the element-wise derivative of the activation function at layer \(l-1\).

Since the activation derivative \(\varphi'_l(\bm{Z}_{l-1})\) scales each component individually without altering the overall direction, it follows directly that:
\begin{equation*}
\mathrm{direction}(\bm{D}_{l-1}) = \mathrm{direction}\left(\mathbf{J}_l^\top \bm{D}_{l}\right).
\end{equation*}

Thus, the dominance of gradient direction by linear transformations is universally applicable across diverse neural network architectures, including CNNs, RNNs, transformers, and graph-based networks.
\end{proof}

\begin{proposition}[Feasibility of Forward-Backward Decoupling]
Breaking the conventional symmetry between forward activation functions and backward gradients does not hinder neural network learning.
\end{proposition}

\begin{proof}
Consider the backward gradient propagation rule previously defined in Eq.~\eqref{eq:D_update}. In an untied backward configuration, we redefine the gradient at layer \( l-1 \) as:
\begin{equation}
\bm{D}_{l-1}^{untied} = \left(\mathbf{J}_{l}\right)^\top \bm{D}_{l} \odot g(\bm{Z}_{l-1}), \quad \text{with} \quad g(\bm{Z}_{l-1}) > 0,
\end{equation}
where \(\left(\mathbf{J}_{l}\right)^\top\) denotes the linear Jacobian or equivalent linear operator associated with layer \( l \), and \( g(\bm{Z}_{l-1}) \) is an element-wise, strictly positive function that replaces the standard activation derivative.

The directional consistency remains intact because the element-wise function \( g(\bm{Z}_{l-1}) \) is strictly positive. Thus, for each component \( j \), we have:
\begin{equation*}
\mathrm{sign}\left((\bm{D}_{l-1}^{untied})_j\right) = \mathrm{sign}\left(\left[(\mathbf{J}_{l})^\top \bm{D}_{l}\right]_j\right).
\end{equation*}

Because successful learning in gradient-based optimization predominantly depends on maintaining correct gradient directions rather than precise magnitudes, replacing the conventional activation derivatives with an alternative, strictly positive function \( g(\bm{Z}_{l-1}) \) does not impair the learning process. Consequently, the theoretical feasibility of forward-backward decoupling is established.
\end{proof}

Building upon the theoretical propositions established above, we now rigorously interpret Eq. \eqref{eq:update-notime}, specifically Eqs. \eqref{eq:W_update} and \eqref{eq:D_update}. To do so effectively, it is essential to explicitly analyze the structural properties governing gradient propagation in neural networks. We thus examine a multi-layer perceptron (MLP) as a foundational architecture. Although simpler compared to contemporary neural network architectures, the MLP effectively captures the essential gradient dynamics at the heart of our theoretical discussion.
}

Considering the explicit form of the function \( G_l(\bm{W}_l, \bm{Z}_{l-1}) \):

\begin{equation}
G_l(\bm{W}_l, \bm{Z}_{l-1}) = f_l(\bm{W}_l \cdot [\bm{Z}_{l-1}; 1]),
\label{eq:MLP_activation}
\end{equation}

we see the implications of our theoretical propositions clearly. Activation functions \( f_l \) are typically selected for their monotonicity and continuity, though exceptions such as ReLU remain effective despite non-differentiability at certain points.

Geometrically, each neuron in a network defines a hyperplane that partitions the input space into two half-spaces. Explicitly expanding this definition, we obtain the general equation describing the hyperplane associated with neuron $h$ at layer $l$:
\begin{equation}
\sum_{i=1}^{n_{l-1}} W_{l,h i} Z_{l-1,i} + W_{l,h 0} = 0,
\end{equation}
where $n_{l-1}$ denotes the number of neurons in the preceding layer, $Z_{l-1,i}$ represents the activation from neuron $i$ at layer $l-1$, and $W_{l,h 0}$ is the bias term associated with neuron $h$ at layer $l$. This hyperplane is explicitly characterized by its unit normal vector:
\begin{equation}
\hat{\bm{W}}_{l,h} = \frac{\bm{W}_{l,h}}{\|\bm{W}_{l,h}\|},
\end{equation}
which determines its orientation, and its orthogonal distance from the origin given by:
\begin{equation}
\vartheta_0 = \frac{|W_{l,h 0}|}{\|\bm{W}_{l,h}\|}.
\end{equation}

Given fixed weights $W_{l,h i}$, varying only the bias $W_{l,h 0}$ yields a family of parallel hyperplanes whose orientation is determined solely by $\hat{\bm{W}}_{l,h}$, thus by $W_{l,h i},\, i \geq 1$. Each hyperplane defined by the equation:
\begin{equation}
\sum_{i=1}^{n_{l-1}} W_{l,h i} Z_{l-1,i} + W_{l,h 0} = c,
\end{equation}
corresponds to a unique linear output $c$. The activation function then maps this linear output $c$ to an activation value $f(c)$. For any given input point $\bm{Z}_{l-1}$, the neuron calculates this linear output, precisely identifying the hyperplane associated with that linear output. The activation function subsequently produces the neuron's final output.

Weight updates performed through backpropagation (as described in Eq.~\eqref{eq:W_update}) explicitly induce two primary geometric transformations of these hyperplanes:

\begin{enumerate}
    \item \textbf{Translation}: Primarily driven by bias updates ($W_{l,h0}$), shifting the hyperplane along its normal direction without changing its orientation.

    \item \textbf{Rotation}: Caused by updates to weights excluding biases ($W_{l,h i},\, i \geq 1$), altering the orientation of the hyperplane around the origin.
\end{enumerate}

Thus, through weight adjustments, neural networks dynamically reposition and reorient these hyperplanes, effectively adapting decision boundaries during the learning process.

\begin{figure}[!t]
    \centering
    \subfloat[3D visualization showing rotations and translations of hyperplanes due to weight adjustments.]{
        \includegraphics[width=0.49\textwidth]{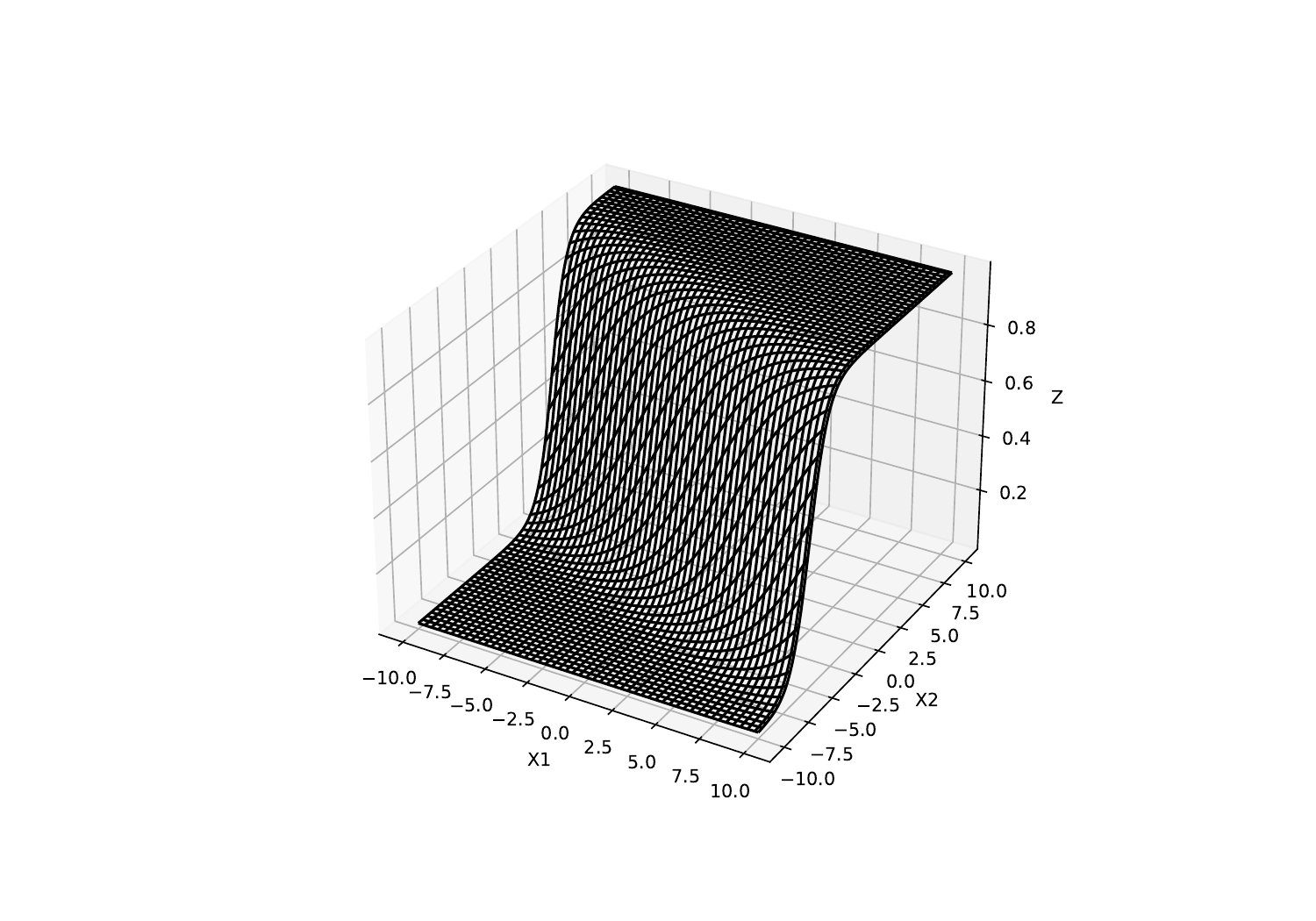}
        \includegraphics[width=0.49\textwidth]{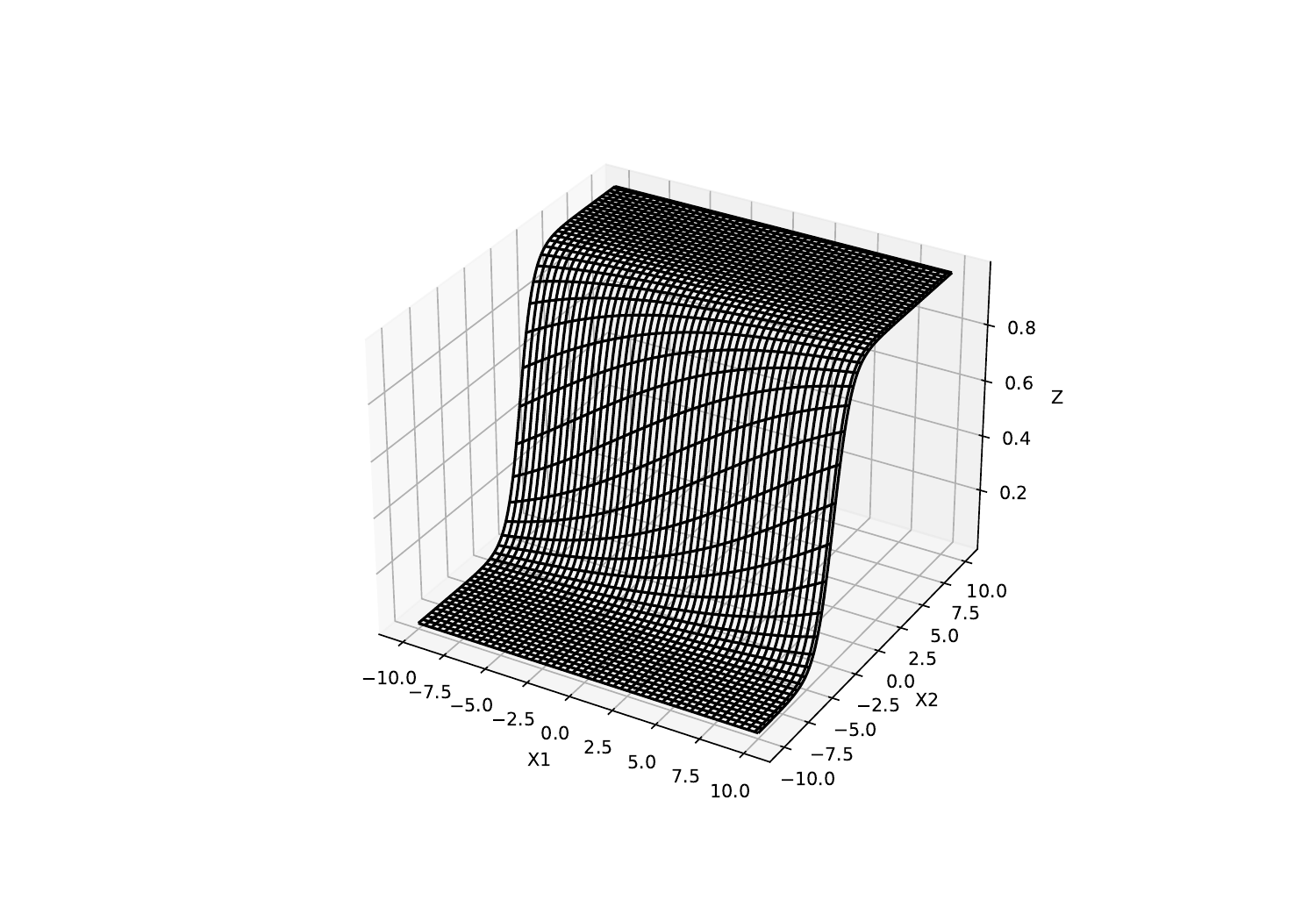}
        \label{fig:3d_comparison}
    }
    \vspace{1mm}
    \subfloat[2D visualization clearly demonstrating hyperplane adjustments from original (left) to updated weights (right).]{
        \includegraphics[width=0.45\textwidth]{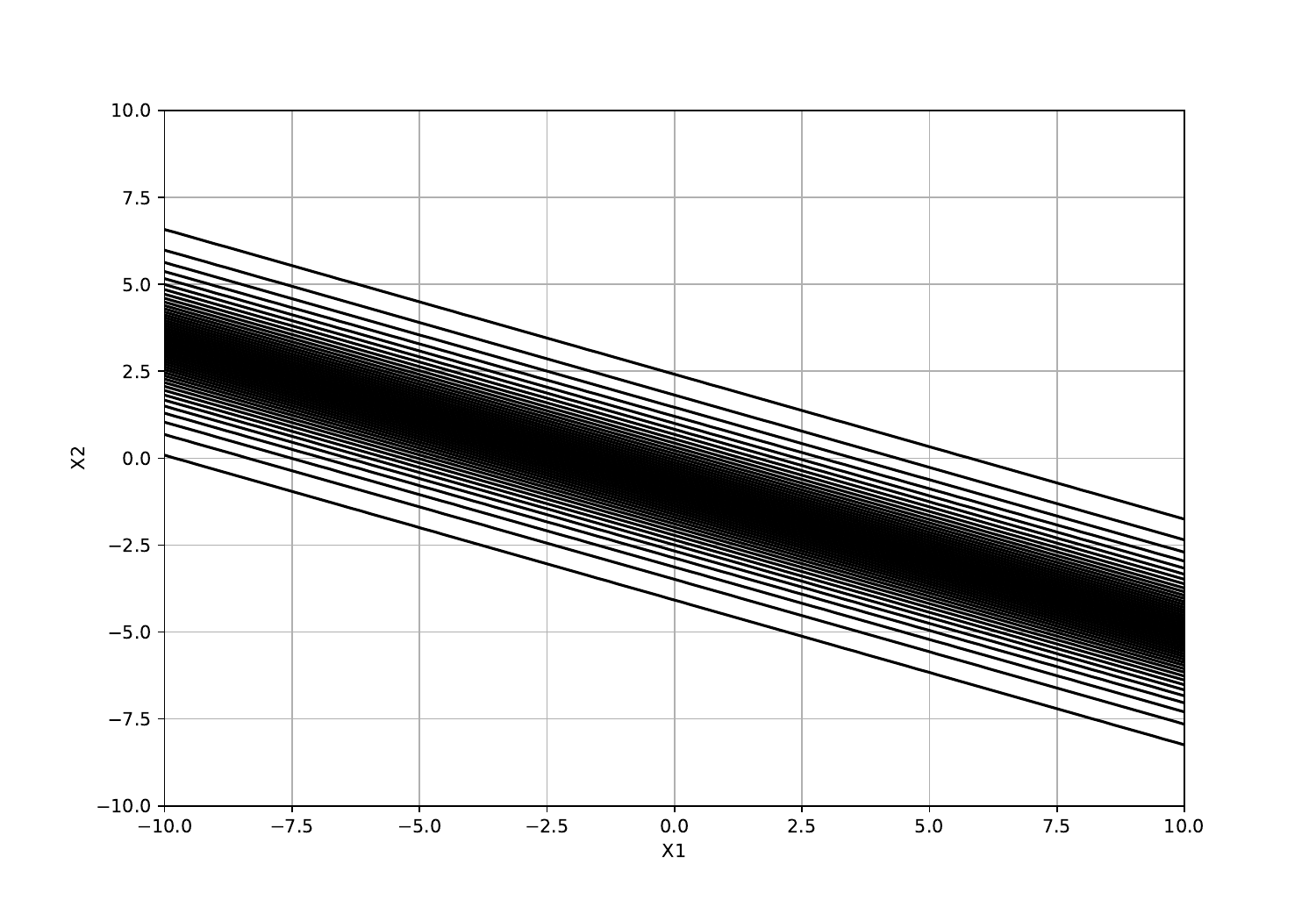}
        \hspace{1cm}
        \includegraphics[width=0.45\textwidth]{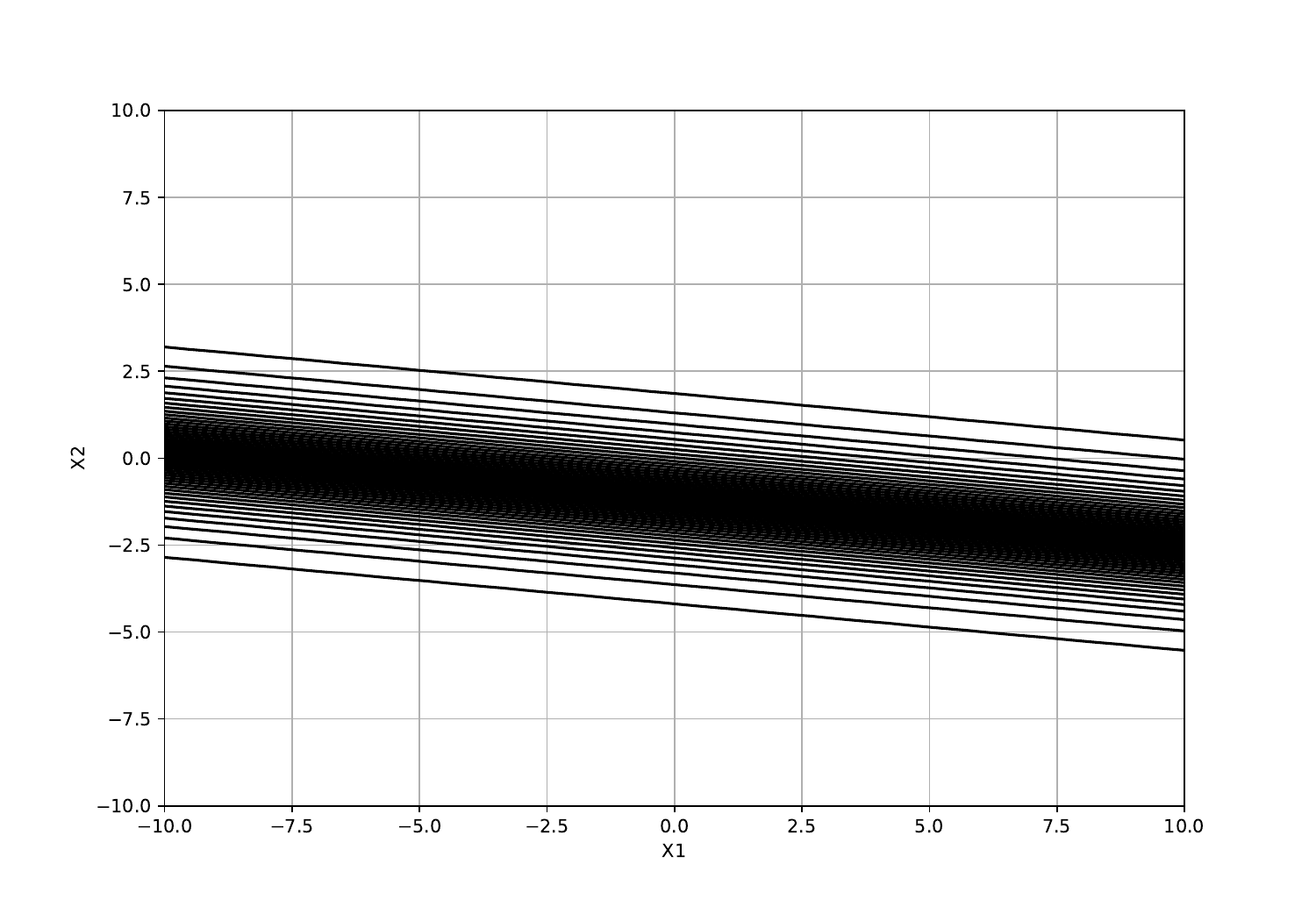}
        \label{fig:2d_comparison}
    }

    \caption{Illustration of weight updates' impact on hyperplane configurations.}
    \label{fig:weight_adjustment_visualization}
\end{figure}

Figure~\ref{fig:weight_adjustment_visualization} visually demonstrates these transformations. Specifically, the left panels depict original activation functions and hyperplanes, whereas the right panels illustrate the consequences of weight updates—rotations and translations that reshape decision boundaries.

These visualizations explicitly demonstrate the theoretical insights from Proposition 1: while the magnitude of updates depends on gradients and errors, the activation vector from the preceding layer exclusively determines their directionality.

To further elucidate and substantiate the theoretical propositions established above, we first analyze these concepts within the context of a simplified, yet representative neural architecture.: the Single Unit Classifier (SUC), depicted in Fig.~\ref{fig:SUC}. The forward and backward flows characterizing the training process of the SUC are clearly illustrated in Fig.~\ref{fig:SUC-process}, allowing us to concretely examine the practical impact of decoupling forward activation functions from backward gradient computations.

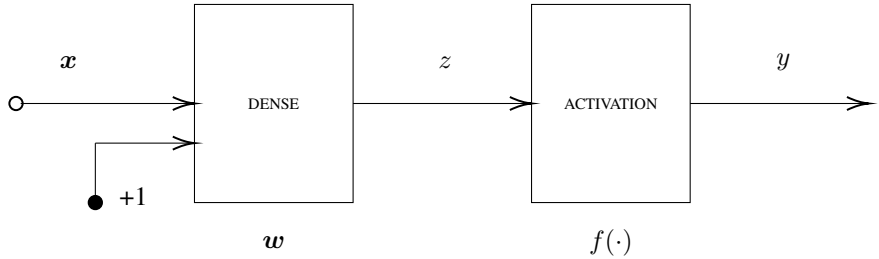
\begin{figure}[t]
\centering
\begin{tikzpicture}[x=0.75pt,y=0.75pt,yscale=-1,xscale=1]

\draw   (100,10) -- (180,10) -- (180,110) -- (100,110) -- cycle ;
\draw   (270,10) -- (350,10) -- (350,110) -- (270,110) -- cycle ;
\draw    (12.35,60) -- (98,60) ;
\draw [shift={(100,60)}, rotate = 180] [color={rgb, 255:red, 0; green, 0; blue, 0 }  ][line width=0.75]    (10.93,-3.29) .. controls (6.95,-1.4) and (3.31,-0.3) .. (0,0) .. controls (3.31,0.3) and (6.95,1.4) .. (10.93,3.29)   ;
\draw [shift={(10,60)}, rotate = 0] [color={rgb, 255:red, 0; green, 0; blue, 0 }  ][line width=0.75]      (0, 0) circle [x radius= 3.35, y radius= 3.35]   ;
\draw    (180,60) -- (268,60) ;
\draw [shift={(270,60)}, rotate = 180] [color={rgb, 255:red, 0; green, 0; blue, 0 }  ][line width=0.75]    (10.93,-3.29) .. controls (6.95,-1.4) and (3.31,-0.3) .. (0,0) .. controls (3.31,0.3) and (6.95,1.4) .. (10.93,3.29)   ;
\draw    (350,60) -- (438,60) ;
\draw [shift={(440,60)}, rotate = 180] [color={rgb, 255:red, 0; green, 0; blue, 0 }  ][line width=0.75]    (10.93,-3.29) .. controls (6.95,-1.4) and (3.31,-0.3) .. (0,0) .. controls (3.31,0.3) and (6.95,1.4) .. (10.93,3.29)   ;
\draw    (50,80) -- (98,80) ;
\draw [shift={(100,80)}, rotate = 180] [color={rgb, 255:red, 0; green, 0; blue, 0 }  ][line width=0.75]    (10.93,-3.29) .. controls (6.95,-1.4) and (3.31,-0.3) .. (0,0) .. controls (3.31,0.3) and (6.95,1.4) .. (10.93,3.29)   ;
\draw    (50,80) -- (50,110) ;
\draw [shift={(50,110)}, rotate = 90] [color={rgb, 255:red, 0; green, 0; blue, 0 }  ][fill={rgb, 255:red, 0; green, 0; blue, 0 }  ][line width=0.75]      (0, 0) circle [x radius= 3.35, y radius= 3.35]   ;

\draw (36.5,39) node   [align=left] {$\bm{x}$};
\draw (140,60) node   [align=left] {{\tiny DENSE}};
\draw (310,60) node   [align=left] {{\tiny ACTIVATION}};
\draw (61,101) node [anchor=north west][inner sep=0.75pt]   [align=left] {+1};
\draw (140,130) node    {$\bm{w}$};
\draw (310,130) node    {$f(\cdot)$};
\draw (226.5,39) node   [align=left] {$z$};
\draw (397,39) node    {$y$};

\end{tikzpicture}
\caption{Schematic representation of the Single Unit Classifier (SUC).}\label{fig:SUC} 
\end{figure}

A Single Unit Classifier (SUC) comprises two fundamental blocks: a \emph{dense block} and an \emph{activation block}. Each of these blocks plays a distinct role in shaping the gradient dynamics during training:
\begin{enumerate}
\item \textbf{Dense Block:} This block implements a linear transformation. Given an input vector \( \bm{x} \) and an augmented weight vector \( \bm{w} \) (including bias), the dense block computes its output \( z \) as:
\begin{equation}
z = \bm{w}^T \begin{bmatrix} 1 \\[6pt] \bm{x} \end{bmatrix}.
\end{equation}
\item \textbf{Activation Block:} Subsequently, the output \( z \) from the dense block undergoes a nonlinear transformation through an activation function \( f \), yielding the final scalar output \( y \):
\begin{equation*}
y = f(z).
\end{equation*}
\end{enumerate}

Considering the chain rule of differentiation as introduced in the theoretical foundations, the derivative of \( y \) with respect to the input vector \( \bm{x} \) can be explicitly expressed as:
\begin{equation}
\frac{\partial y}{\partial \bm{x}} = \frac{\partial y}{\partial z} \cdot \frac{\partial z}{\partial \bm{x}}.
\end{equation}

In this expression, \( \frac{\partial y}{\partial z} \) is determined exclusively by the derivative \( f'(z) \) of the chosen activation function, while \( \frac{\partial z}{\partial \bm{x}} \) is dictated by the weight vector \( \bm{w} \), reflecting the linearity of the dense block. Due to this linear relationship, the directional component of the gradient \( \frac{\partial y}{\partial \bm{x}} \) aligns exactly with the augmented weight vector \( \bm{w} \). Explicitly, its unit direction vector (\emph{versor}) is given by:
\begin{equation}
\hat{\bm{v}} = \frac{\bm{w}}{\|\bm{w}\|_2},
\end{equation}
where \( \|\bm{w}\|_2 \) denotes the L2 norm of the weight vector \( \bm{w} \). Meanwhile, the magnitude of the gradient, capturing the combined influence of both the activation function's derivative and the weight vector's magnitude, is:
\begin{equation}
f'(z) \times \|\bm{w}\|_2.
\end{equation}

From this explicit analysis, consistent with Propositions 1–3, we observe a clear separation of roles: the nonlinearity introduced by the activation function does not alter the gradient direction but rather exclusively modulates its magnitude. Thus, methods aiming to improve training dynamics, such as adaptive learning rates or gradient normalization strategies discussed in Section 2, primarily affect the gradient magnitude. This understanding opens intriguing possibilities: by explicitly recognizing this directional invariance, one could deliberately modify the gradient magnitude independently, introducing alternative functions or even stochastic variations in the backward propagation without compromising the gradient's directional integrity.

\begin{figure}[!ht]
\centering
\resizebox{\textwidth}{!}{
\begin{tikzpicture}[x=0.75pt,y=0.75pt,yscale=-1,xscale=1]

\draw    (180,60) -- (180,78) ;
\draw [shift={(180,80)}, rotate = 270] [color={rgb, 255:red, 0; green, 0; blue, 0 }  ][line width=0.75]    (10.93,-3.29) .. controls (6.95,-1.4) and (3.31,-0.3) .. (0,0) .. controls (3.31,0.3) and (6.95,1.4) .. (10.93,3.29)   ;
\draw    (180,100) -- (180,118) ;
\draw [shift={(180,120)}, rotate = 270] [color={rgb, 255:red, 0; green, 0; blue, 0 }  ][line width=0.75]    (10.93,-3.29) .. controls (6.95,-1.4) and (3.31,-0.3) .. (0,0) .. controls (3.31,0.3) and (6.95,1.4) .. (10.93,3.29)   ;
\draw    (130,130) -- (148,130) ;
\draw [shift={(150,130)}, rotate = 180] [color={rgb, 255:red, 0; green, 0; blue, 0 }  ][line width=0.75]    (10.93,-3.29) .. controls (6.95,-1.4) and (3.31,-0.3) .. (0,0) .. controls (3.31,0.3) and (6.95,1.4) .. (10.93,3.29)   ;
\draw    (180,140) -- (180,158) ;
\draw [shift={(180,160)}, rotate = 270] [color={rgb, 255:red, 0; green, 0; blue, 0 }  ][line width=0.75]    (10.93,-3.29) .. controls (6.95,-1.4) and (3.31,-0.3) .. (0,0) .. controls (3.31,0.3) and (6.95,1.4) .. (10.93,3.29)   ;
\draw    (180,180) -- (180,198) ;
\draw [shift={(180,200)}, rotate = 270] [color={rgb, 255:red, 0; green, 0; blue, 0 }  ][line width=0.75]    (10.93,-3.29) .. controls (6.95,-1.4) and (3.31,-0.3) .. (0,0) .. controls (3.31,0.3) and (6.95,1.4) .. (10.93,3.29)   ;
\draw    (180,220) -- (180,238) ;
\draw [shift={(180,240)}, rotate = 270] [color={rgb, 255:red, 0; green, 0; blue, 0 }  ][line width=0.75]    (10.93,-3.29) .. controls (6.95,-1.4) and (3.31,-0.3) .. (0,0) .. controls (3.31,0.3) and (6.95,1.4) .. (10.93,3.29)   ;
\draw    (180,260) .. controls (180.48,302.88) and (182.96,298.45) .. (198.05,298.03) ;
\draw [shift={(200,298)}, rotate = 180] [color={rgb, 255:red, 0; green, 0; blue, 0 }  ][line width=0.75]    (10.93,-3.29) .. controls (6.95,-1.4) and (3.31,-0.3) .. (0,0) .. controls (3.31,0.3) and (6.95,1.4) .. (10.93,3.29)   ;
\draw    (270,360) -- (270,312) ;
\draw [shift={(270,310)}, rotate = 450] [color={rgb, 255:red, 0; green, 0; blue, 0 }  ][line width=0.75]    (10.93,-3.29) .. controls (6.95,-1.4) and (3.31,-0.3) .. (0,0) .. controls (3.31,0.3) and (6.95,1.4) .. (10.93,3.29)   ;
\draw    (580.1,262.05) .. controls (581.72,294.9) and (578.78,300.65) .. (560,300) ;
\draw [shift={(580,260)}, rotate = 86.97] [color={rgb, 255:red, 0; green, 0; blue, 0 }  ][line width=0.75]    (10.93,-3.29) .. controls (6.95,-1.4) and (3.31,-0.3) .. (0,0) .. controls (3.31,0.3) and (6.95,1.4) .. (10.93,3.29)   ;
\draw    (580,240) -- (580,222) ;
\draw [shift={(580,220)}, rotate = 450] [color={rgb, 255:red, 0; green, 0; blue, 0 }  ][line width=0.75]    (10.93,-3.29) .. controls (6.95,-1.4) and (3.31,-0.3) .. (0,0) .. controls (3.31,0.3) and (6.95,1.4) .. (10.93,3.29)   ;
\draw    (510,210) -- (528,210) ;
\draw [shift={(530,210)}, rotate = 180] [color={rgb, 255:red, 0; green, 0; blue, 0 }  ][line width=0.75]    (10.93,-3.29) .. controls (6.95,-1.4) and (3.31,-0.3) .. (0,0) .. controls (3.31,0.3) and (6.95,1.4) .. (10.93,3.29)   ;
\draw  [dash pattern={on 0.84pt off 2.51pt}]  (210,210) -- (440,210) ;
\draw    (200,170) .. controls (259.4,170.99) and (387.41,211.18) .. (448.18,210.05) ;
\draw [shift={(450,210)}, rotate = 538.0899999999999] [color={rgb, 255:red, 0; green, 0; blue, 0 }  ][line width=0.75]    (10.93,-3.29) .. controls (6.95,-1.4) and (3.31,-0.3) .. (0,0) .. controls (3.31,0.3) and (6.95,1.4) .. (10.93,3.29)   ;
\draw    (580,200) -- (580,182) ;
\draw [shift={(580,180)}, rotate = 450] [color={rgb, 255:red, 0; green, 0; blue, 0 }  ][line width=0.75]    (10.93,-3.29) .. controls (6.95,-1.4) and (3.31,-0.3) .. (0,0) .. controls (3.31,0.3) and (6.95,1.4) .. (10.93,3.29)   ;
\draw    (580,160) -- (580,142) ;
\draw [shift={(580,140)}, rotate = 450] [color={rgb, 255:red, 0; green, 0; blue, 0 }  ][line width=0.75]    (10.93,-3.29) .. controls (6.95,-1.4) and (3.31,-0.3) .. (0,0) .. controls (3.31,0.3) and (6.95,1.4) .. (10.93,3.29)   ;
\draw    (630,130) -- (612,130) ;
\draw [shift={(610,130)}, rotate = 360] [color={rgb, 255:red, 0; green, 0; blue, 0 }  ][line width=0.75]    (10.93,-3.29) .. controls (6.95,-1.4) and (3.31,-0.3) .. (0,0) .. controls (3.31,0.3) and (6.95,1.4) .. (10.93,3.29)   ;
\draw    (509.93,142.19) .. controls (509.04,170) and (510.53,170) .. (570,170) ;
\draw [shift={(510,140)}, rotate = 91.91] [color={rgb, 255:red, 0; green, 0; blue, 0 }  ][line width=0.75]    (10.93,-3.29) .. controls (6.95,-1.4) and (3.31,-0.3) .. (0,0) .. controls (3.31,0.3) and (6.95,1.4) .. (10.93,3.29)   ;
\draw    (580,80) -- (580,50) ;
\draw [shift={(580,50)}, rotate = 270] [color={rgb, 255:red, 0; green, 0; blue, 0 }  ][fill={rgb, 255:red, 0; green, 0; blue, 0 }  ][line width=0.75]      (0, 0) circle [x radius= 3.35, y radius= 3.35]   ;
\draw    (580,120) -- (580,102) ;
\draw [shift={(580,100)}, rotate = 450] [color={rgb, 255:red, 0; green, 0; blue, 0 }  ][line width=0.75]    (10.93,-3.29) .. controls (6.95,-1.4) and (3.31,-0.3) .. (0,0) .. controls (3.31,0.3) and (6.95,1.4) .. (10.93,3.29)   ;
\draw    (200,90) .. controls (259.4,90.99) and (348.2,131.18) .. (408.19,130.05) ;
\draw [shift={(410,130)}, rotate = 538.0899999999999] [color={rgb, 255:red, 0; green, 0; blue, 0 }  ][line width=0.75]    (10.93,-3.29) .. controls (6.95,-1.4) and (3.31,-0.3) .. (0,0) .. controls (3.31,0.3) and (6.95,1.4) .. (10.93,3.29)   ;
\draw    (10,88.5) -- (160,88.5)(10,91.5) -- (160,91.5) ;
\draw    (270,88.5) -- (560,88.5)(270,91.5) -- (560,91.5) ;
\draw    (600,88.5) -- (660,88.5)(600,91.5) -- (660,91.5) ;
\draw    (270,168.5) -- (500,168.5)(270,171.5) -- (500,171.5) ;
\draw    (10,168.5) -- (160,168.5)(10,171.5) -- (160,171.5) ;
\draw    (10,248.5) -- (160,248.5)(10,251.5) -- (160,251.5) ;
\draw    (600,168.5) -- (660,168.5)(600,171.5) -- (660,171.5) ;
\draw    (600,248.5) -- (660,248.5)(600,251.5) -- (660,251.5) ;
\draw    (200,248.5) -- (560,248.5)(200,251.5) -- (560,251.5) ;
\draw    (180,260) -- (180,358) ;
\draw [shift={(180,360)}, rotate = 270] [color={rgb, 255:red, 0; green, 0; blue, 0 }  ][line width=0.75]    (10.93,-3.29) .. controls (6.95,-1.4) and (3.31,-0.3) .. (0,0) .. controls (3.31,0.3) and (6.95,1.4) .. (10.93,3.29)   ;
\draw    (10,328.5) -- (160,328.5)(10,331.5) -- (160,331.5) ;
\draw    (200,328.5) -- (250,328.5)(200,331.5) -- (250,331.5) ;
\draw    (290,328.5) -- (660,328.5)(290,331.5) -- (660,331.5) ;
\draw    (270,360) .. controls (269.5,313.75) and (269.83,320.67) .. (330,320) ;
\draw    (180,260) .. controls (177.17,310.67) and (519.17,238.67) .. (520,280) ;
\draw    (520,280) -- (520,288) ;
\draw [shift={(520,290)}, rotate = 270] [color={rgb, 255:red, 0; green, 0; blue, 0 }  ][line width=0.75]    (10.93,-3.29) .. controls (6.95,-1.4) and (3.31,-0.3) .. (0,0) .. controls (3.31,0.3) and (6.95,1.4) .. (10.93,3.29)   ;
\draw  [dash pattern={on 0.84pt off 2.51pt}]  (360,300) -- (410,300) ;
\draw    (497.97,311.21) .. controls (464.62,330.4) and (418.04,323.92) .. (330,320) ;
\draw [shift={(500,310)}, rotate = 148.33] [color={rgb, 255:red, 0; green, 0; blue, 0 }  ][line width=0.75]    (10.93,-3.29) .. controls (6.95,-1.4) and (3.31,-0.3) .. (0,0) .. controls (3.31,0.3) and (6.95,1.4) .. (10.93,3.29)   ;

\draw (182,49) node   [align=left] {$\displaystyle \mathbf{X}$};
\draw (114,129) node    {$\mathbf{W}$};
\draw (180.5,89) node    {$\mathbf{Z}_{0}$};
\draw (178.5,131) node    {$\mathbf{WZ}_{0}$};
\draw (183,209) node    {$\varphi (\mathbf{Z}_{1})$};
\draw (180.5,251) node    {$\mathbf{Z}_{2}$};
\draw (284,297) node    {$\mathcal{L} =\frac{1}{2m} \ \mathbf{1}^{T} \cdot (\mathbf{Y} -\mathbf{Z}_{2})^{\circ 2}$};
\draw (487.5,296) node    {$\nabla \mathcal{L} =-\frac{1}{m}(\mathbf{Y} -\mathbf{Z}_{2})$};
\draw (181.5,169) node    {$\mathbf{Z}_{1}$};
\draw (582.5,251) node    {$\mathbf{D}_{2}$};
\draw (481.5,208.5) node    {$\varphi ^{'}(\mathbf{Z}_{1})$};
\draw (584,207) node    {$\mathbf{\varphi ^{'}( Z_{1}) \circ D}_{2}$};
\draw (582.5,171) node    {$\mathbf{D}_{1}$};
\draw (576.5,131) node    {$\mathbf{W}^{T}\mathbf{D}_{1}$};
\draw (654,131) node    {$\mathbf{W}$};
\draw (481,129.5) node    {$\Delta \mathbf{W} =-\eta \mathbf{D}_{1}\mathbf{Z}^{T}_{0}$};
\draw (582.5,91) node    {$\mathbf{D}_{0}$};
\draw (11,132) node [anchor=west] [inner sep=0.75pt]  [font=\small] [align=left] {LAYER 1\\(FC)};
\draw (11,212) node [anchor=west] [inner sep=0.75pt]  [font=\small] [align=left] {LAYER 2\\(ACTIVATION)};
\draw (11,56) node [anchor=west] [inner sep=0.75pt]  [font=\small] [align=left] {INPUT};
\draw (11,291) node [anchor=west] [inner sep=0.75pt]  [font=\small] [align=left] {LOSS};
\draw (181.5,373) node   [align=left] {$\displaystyle \mathring{\mathbf{Y}}$};
\draw (9,375) node [anchor=west] [inner sep=0.75pt]  [font=\small] [align=left] {OUTPUT};
\draw (268.5,376) node   [align=left] {$\displaystyle \mathbf{Y}$};
\draw (135,12) node [anchor=north west][inner sep=0.75pt]   [align=left] {FORWARD};
\draw (575,20.5) node   [align=left] {BACKWARD};

\end{tikzpicture}
}
\caption{Block diagram showing the training of SUC, where $\mathring{\textbf{Y}}$ is the network output, $\textbf{Y}$ the target, $\mathbf{W}_i$, $\varphi_i(\cdot)$, $\textbf{Z}_i$ are respectively the weighting matrix, the activation function and the output at the $i$-th layers, $\textbf{D}_i$ is the delta factor used for updating the weights at the $i$-th layer.}\label{fig:SUC-process} 
\end{figure}
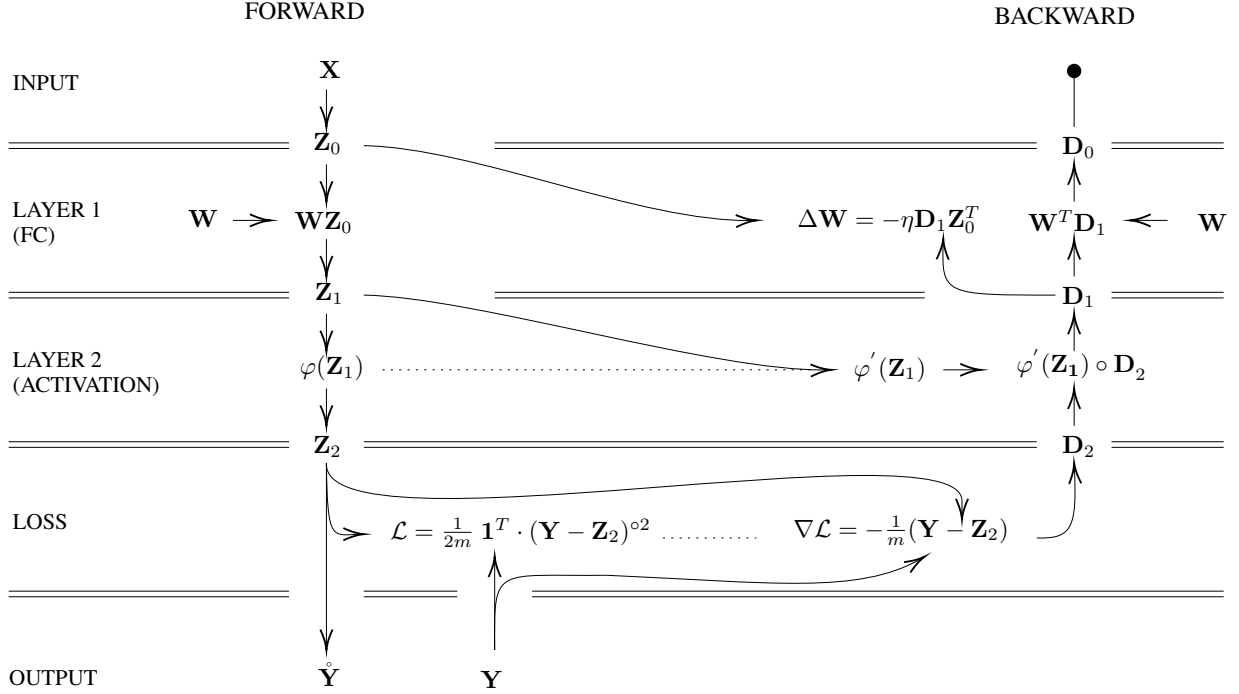

These observations remain valid even when batching is employed, as depicted in Fig.~\ref{fig:SUC-process}. Initially, consider a mini-batch size \( m=1 \). Under this scenario, the activation potential, the target, and the network's output are scalar quantities, i.e., \( \bm{Z}_1 = \zeta \), \( \bm{Y}=\gamma \), and \( \mathring{\bm{Y}}=\bar{\gamma} \in \mathbb{R} \). Consequently, the loss function simplifies to \( \mathcal{L}=\frac{1}{2}(\gamma- \bar{\gamma})^2 \) with gradient \( \nabla \mathcal{L} = (\bar{\gamma}- \gamma) \). Moreover, the derivative of the activation function reduces to a scalar, \( \varphi'(\bm{Z}_1) = \varphi'(\zeta) \). Thus, the scalar local gradient term \( \bm{D}_1 = \varphi'(\zeta)(\bar{\gamma}-\gamma) \) explicitly affects only the magnitude of the parameter update \( \Delta \bm{W} \), but not its direction. Formally:

\begin{equation}
    \Delta \bm{W} = -\lVert \Delta \bm{W} \rVert \hat{u}_w,
\end{equation}

where the unit vector indicating the direction of weight variation is explicitly given by:

\begin{equation}
    \hat{u}_w = \frac{\bm{Z}_0}{\lVert \bm{Z}_0 \rVert} = \frac{\bm{X}}{\lVert \bm{X} \rVert}.
\end{equation}

The explicit magnitude of this variation is:

\begin{equation}
    \lVert \Delta \bm{W} \rVert = \eta \lVert \bm{Z}_0 \rVert |\bm{D}_1| = \eta \lVert \bm{X} \rVert |\varphi'(\zeta)| |\gamma - \bar{\gamma}|.
\end{equation}

Therefore, the gradient of the activation function \( \varphi'(\zeta) \) does not influence the direction of the weight update \( \Delta \bm{W} \). Hence, the magnitude of the weight update can be modulated through alternative criteria, independently of the gradient of the activation function. This property opens possibilities for training architectures such as a SUC with activation functions that typically exhibit near-zero gradients, e.g., the Heaviside step function.

The above considerations naturally generalize beyond \( m=1 \). When the mini-batch comprises multiple samples (\( m > 1 \)), the total weight update is:

\begin{equation*}
\Delta \bm{W} = \sum_{k=1}^{m} \Delta \bm{W}_k,
\end{equation*}
with each sample's contribution being explicitly:
\begin{equation}
    \Delta \mathbf{W}_{k} = -\eta \varphi'(\zeta_k) (\gamma_k - \bar{\gamma}_k) \mathbf{Z}_{0;k}^T.
\end{equation}

Since each individual update \( \Delta \mathbf{W}_{k} \) maintains the direction dictated by \( \mathbf{Z}_{0;k} \), the activation derivative \( \varphi'(\zeta_k) \) influences solely the magnitude, preserving the directional invariance as highlighted in Proposition~1.

This directional invariance holds consistently, even when the network output dimension \( p > 1 \). For such cases, the weight update matrix \( \Delta \mathbf{W} \in \mathbb{R}^{p \times q} \) is explicitly expressed as:

\begin{equation}
\Delta \mathbf{W} = -\eta \mathbf{D}_1 \mathbf{Z}_0^T = 
\begin{pmatrix}
  \sum_{k=1}^{m} -\eta \mathbf{D}_{1;1,k} \mathbf{Z}_{0;k,1} & \cdots & \sum_{k=1}^{m} -\eta \mathbf{D}_{1;1,k} \mathbf{Z}_{0;k,q}\\
  \vdots & \ddots & \vdots \\
  \sum_{k=1}^{m} -\eta \mathbf{D}_{1;p,k} \mathbf{Z}_{0;k,1} & \cdots & \sum_{k=1}^{m} -\eta \mathbf{D}_{1;p,k} \mathbf{Z}_{0;k,q}
\end{pmatrix},
\end{equation}

further demonstrating that the directionality of the weight updates remains unaffected by the activation derivative \( \varphi' \).

The directional invariance observed in these derivations, consistent with Propositions~2 and 3, extends beyond MLP architectures and is relevant to contemporary deep neural networks composed of modular building blocks. The generalization to broader classes of architectures relies on the following explicit observations:

\begin{enumerate}
\item \textbf{Layer Decomposition:} Architectures with explicitly defined blocks for linear transformations and activation functions (such as those illustrated in Fig.~\ref{fig:mlp-process}) inherently share the same gradient propagation structure, thus preserving directional invariance.

\item \textbf{Tensor-based Operations:} Advanced neural networks operating on tensor data fundamentally leverage inner products along specific tensor dimensions, maintaining linearity at the neuron level. Architectures like CNNs and ResNets explicitly perform spatially constrained inner products, while RNNs and Transformers perform inner products along temporal or feature dimensions, preserving directional invariance of gradient updates.

\item \textbf{Graph-based Architectures:} Even in complex architectures characterized by non-sequential graph structures, gradient propagation remains explicitly modular. Each module’s gradient updates depend linearly on the errors backpropagated from subsequent modules and activations received from preceding modules. Thus, despite architectural complexity, the core directional invariance principles remain explicitly intact.
\end{enumerate}

Consequently, our theoretical analysis of directional invariance, although explicitly demonstrated within simpler neural architectures, robustly generalizes to contemporary and structurally diverse deep neural networks.

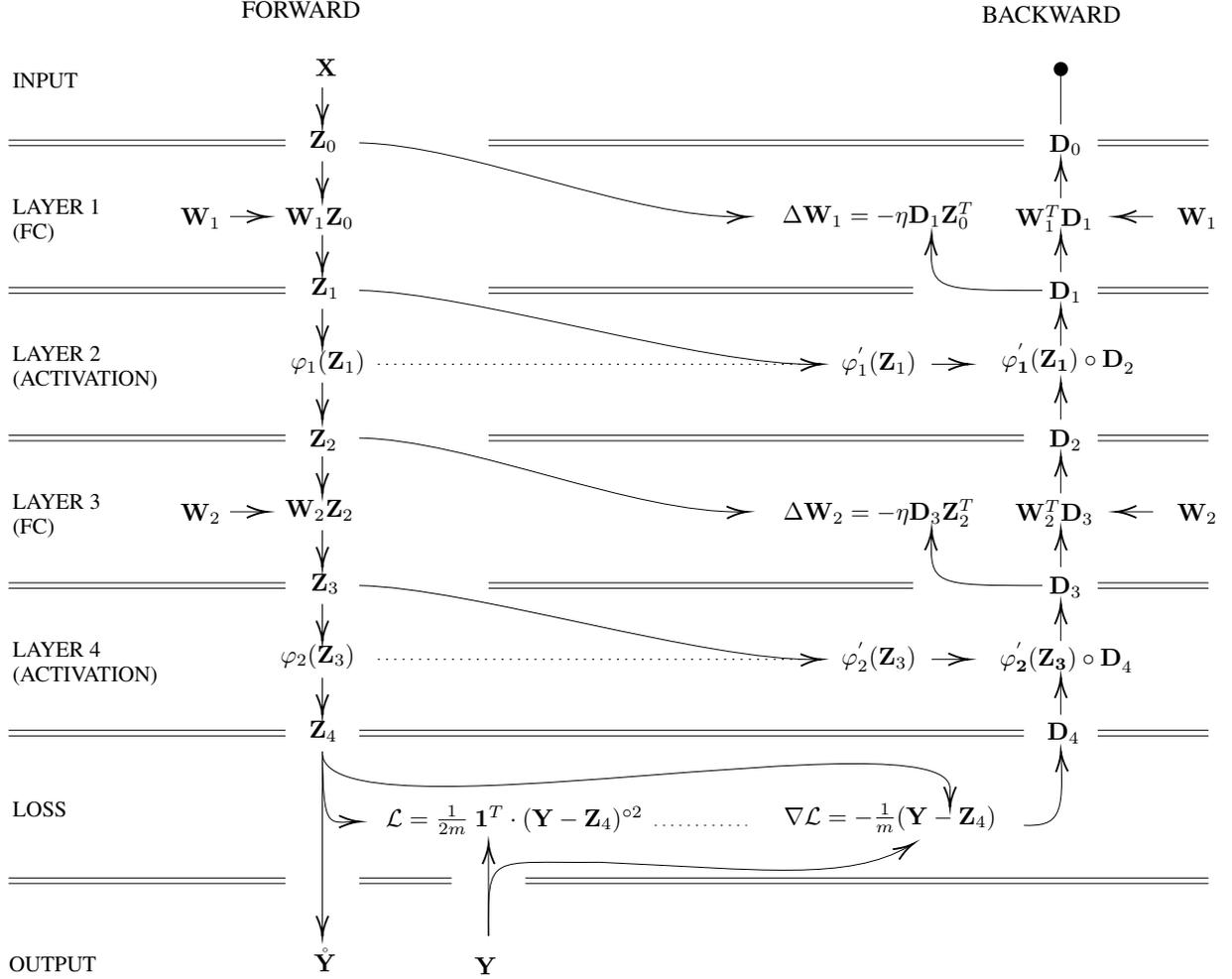
\begin{figure}[!tp]
\centering
\resizebox{\textwidth}{!}{
\begin{tikzpicture}[x=0.75pt,y=0.75pt,yscale=-1,xscale=1]

\draw    (180,60) -- (180,78) ;
\draw [shift={(180,80)}, rotate = 270] [color={rgb, 255:red, 0; green, 0; blue, 0 }  ][line width=0.75]    (10.93,-3.29) .. controls (6.95,-1.4) and (3.31,-0.3) .. (0,0) .. controls (3.31,0.3) and (6.95,1.4) .. (10.93,3.29)   ;
\draw    (180,100) -- (180,118) ;
\draw [shift={(180,120)}, rotate = 270] [color={rgb, 255:red, 0; green, 0; blue, 0 }  ][line width=0.75]    (10.93,-3.29) .. controls (6.95,-1.4) and (3.31,-0.3) .. (0,0) .. controls (3.31,0.3) and (6.95,1.4) .. (10.93,3.29)   ;
\draw    (130,130) -- (148,130) ;
\draw [shift={(150,130)}, rotate = 180] [color={rgb, 255:red, 0; green, 0; blue, 0 }  ][line width=0.75]    (10.93,-3.29) .. controls (6.95,-1.4) and (3.31,-0.3) .. (0,0) .. controls (3.31,0.3) and (6.95,1.4) .. (10.93,3.29)   ;
\draw    (180,140) -- (180,158) ;
\draw [shift={(180,160)}, rotate = 270] [color={rgb, 255:red, 0; green, 0; blue, 0 }  ][line width=0.75]    (10.93,-3.29) .. controls (6.95,-1.4) and (3.31,-0.3) .. (0,0) .. controls (3.31,0.3) and (6.95,1.4) .. (10.93,3.29)   ;
\draw    (180,180) -- (180,198) ;
\draw [shift={(180,200)}, rotate = 270] [color={rgb, 255:red, 0; green, 0; blue, 0 }  ][line width=0.75]    (10.93,-3.29) .. controls (6.95,-1.4) and (3.31,-0.3) .. (0,0) .. controls (3.31,0.3) and (6.95,1.4) .. (10.93,3.29)   ;
\draw    (180,220) -- (180,238) ;
\draw [shift={(180,240)}, rotate = 270] [color={rgb, 255:red, 0; green, 0; blue, 0 }  ][line width=0.75]    (10.93,-3.29) .. controls (6.95,-1.4) and (3.31,-0.3) .. (0,0) .. controls (3.31,0.3) and (6.95,1.4) .. (10.93,3.29)   ;
\draw    (180,260) -- (180,278) ;
\draw [shift={(180,280)}, rotate = 270] [color={rgb, 255:red, 0; green, 0; blue, 0 }  ][line width=0.75]    (10.93,-3.29) .. controls (6.95,-1.4) and (3.31,-0.3) .. (0,0) .. controls (3.31,0.3) and (6.95,1.4) .. (10.93,3.29)   ;
\draw    (180,300) -- (180,318) ;
\draw [shift={(180,320)}, rotate = 270] [color={rgb, 255:red, 0; green, 0; blue, 0 }  ][line width=0.75]    (10.93,-3.29) .. controls (6.95,-1.4) and (3.31,-0.3) .. (0,0) .. controls (3.31,0.3) and (6.95,1.4) .. (10.93,3.29)   ;
\draw    (130,290) -- (148,290) ;
\draw [shift={(150,290)}, rotate = 180] [color={rgb, 255:red, 0; green, 0; blue, 0 }  ][line width=0.75]    (10.93,-3.29) .. controls (6.95,-1.4) and (3.31,-0.3) .. (0,0) .. controls (3.31,0.3) and (6.95,1.4) .. (10.93,3.29)   ;
\draw    (180,340) -- (180,358) ;
\draw [shift={(180,360)}, rotate = 270] [color={rgb, 255:red, 0; green, 0; blue, 0 }  ][line width=0.75]    (10.93,-3.29) .. controls (6.95,-1.4) and (3.31,-0.3) .. (0,0) .. controls (3.31,0.3) and (6.95,1.4) .. (10.93,3.29)   ;
\draw    (180,380) -- (180,398) ;
\draw [shift={(180,400)}, rotate = 270] [color={rgb, 255:red, 0; green, 0; blue, 0 }  ][line width=0.75]    (10.93,-3.29) .. controls (6.95,-1.4) and (3.31,-0.3) .. (0,0) .. controls (3.31,0.3) and (6.95,1.4) .. (10.93,3.29)   ;
\draw    (180,420) .. controls (180.48,462.88) and (182.96,458.45) .. (198.05,458.03) ;
\draw [shift={(200,458)}, rotate = 180] [color={rgb, 255:red, 0; green, 0; blue, 0 }  ][line width=0.75]    (10.93,-3.29) .. controls (6.95,-1.4) and (3.31,-0.3) .. (0,0) .. controls (3.31,0.3) and (6.95,1.4) .. (10.93,3.29)   ;
\draw    (270,520) -- (270,472) ;
\draw [shift={(270,470)}, rotate = 450] [color={rgb, 255:red, 0; green, 0; blue, 0 }  ][line width=0.75]    (10.93,-3.29) .. controls (6.95,-1.4) and (3.31,-0.3) .. (0,0) .. controls (3.31,0.3) and (6.95,1.4) .. (10.93,3.29)   ;
\draw    (580.1,422.05) .. controls (581.72,454.9) and (578.78,460.65) .. (560,460) ;
\draw [shift={(580,420)}, rotate = 86.97] [color={rgb, 255:red, 0; green, 0; blue, 0 }  ][line width=0.75]    (10.93,-3.29) .. controls (6.95,-1.4) and (3.31,-0.3) .. (0,0) .. controls (3.31,0.3) and (6.95,1.4) .. (10.93,3.29)   ;
\draw    (580,400) -- (580,382) ;
\draw [shift={(580,380)}, rotate = 450] [color={rgb, 255:red, 0; green, 0; blue, 0 }  ][line width=0.75]    (10.93,-3.29) .. controls (6.95,-1.4) and (3.31,-0.3) .. (0,0) .. controls (3.31,0.3) and (6.95,1.4) .. (10.93,3.29)   ;
\draw  [dash pattern={on 0.84pt off 2.51pt}]  (210,370) -- (440,370) ;
\draw    (200,330) .. controls (259.4,330.99) and (387.41,371.18) .. (448.18,370.05) ;
\draw [shift={(450,370)}, rotate = 538.0899999999999] [color={rgb, 255:red, 0; green, 0; blue, 0 }  ][line width=0.75]    (10.93,-3.29) .. controls (6.95,-1.4) and (3.31,-0.3) .. (0,0) .. controls (3.31,0.3) and (6.95,1.4) .. (10.93,3.29)   ;
\draw    (510,370) -- (528,370) ;
\draw [shift={(530,370)}, rotate = 180] [color={rgb, 255:red, 0; green, 0; blue, 0 }  ][line width=0.75]    (10.93,-3.29) .. controls (6.95,-1.4) and (3.31,-0.3) .. (0,0) .. controls (3.31,0.3) and (6.95,1.4) .. (10.93,3.29)   ;
\draw    (580,360) -- (580,342) ;
\draw [shift={(580,340)}, rotate = 450] [color={rgb, 255:red, 0; green, 0; blue, 0 }  ][line width=0.75]    (10.93,-3.29) .. controls (6.95,-1.4) and (3.31,-0.3) .. (0,0) .. controls (3.31,0.3) and (6.95,1.4) .. (10.93,3.29)   ;
\draw    (580,320) -- (580,302) ;
\draw [shift={(580,300)}, rotate = 450] [color={rgb, 255:red, 0; green, 0; blue, 0 }  ][line width=0.75]    (10.93,-3.29) .. controls (6.95,-1.4) and (3.31,-0.3) .. (0,0) .. controls (3.31,0.3) and (6.95,1.4) .. (10.93,3.29)   ;
\draw    (580,280) -- (580,262) ;
\draw [shift={(580,260)}, rotate = 450] [color={rgb, 255:red, 0; green, 0; blue, 0 }  ][line width=0.75]    (10.93,-3.29) .. controls (6.95,-1.4) and (3.31,-0.3) .. (0,0) .. controls (3.31,0.3) and (6.95,1.4) .. (10.93,3.29)   ;
\draw    (580,240) -- (580,222) ;
\draw [shift={(580,220)}, rotate = 450] [color={rgb, 255:red, 0; green, 0; blue, 0 }  ][line width=0.75]    (10.93,-3.29) .. controls (6.95,-1.4) and (3.31,-0.3) .. (0,0) .. controls (3.31,0.3) and (6.95,1.4) .. (10.93,3.29)   ;
\draw    (510,210) -- (528,210) ;
\draw [shift={(530,210)}, rotate = 180] [color={rgb, 255:red, 0; green, 0; blue, 0 }  ][line width=0.75]    (10.93,-3.29) .. controls (6.95,-1.4) and (3.31,-0.3) .. (0,0) .. controls (3.31,0.3) and (6.95,1.4) .. (10.93,3.29)   ;
\draw    (630,290) -- (612,290) ;
\draw [shift={(610,290)}, rotate = 360] [color={rgb, 255:red, 0; green, 0; blue, 0 }  ][line width=0.75]    (10.93,-3.29) .. controls (6.95,-1.4) and (3.31,-0.3) .. (0,0) .. controls (3.31,0.3) and (6.95,1.4) .. (10.93,3.29)   ;
\draw    (200,250) .. controls (259.4,250.99) and (348.2,291.18) .. (408.19,290.05) ;
\draw [shift={(410,290)}, rotate = 538.0899999999999] [color={rgb, 255:red, 0; green, 0; blue, 0 }  ][line width=0.75]    (10.93,-3.29) .. controls (6.95,-1.4) and (3.31,-0.3) .. (0,0) .. controls (3.31,0.3) and (6.95,1.4) .. (10.93,3.29)   ;
\draw    (509.93,302.19) .. controls (509.04,330) and (510.53,330) .. (570,330) ;
\draw [shift={(510,300)}, rotate = 91.91] [color={rgb, 255:red, 0; green, 0; blue, 0 }  ][line width=0.75]    (10.93,-3.29) .. controls (6.95,-1.4) and (3.31,-0.3) .. (0,0) .. controls (3.31,0.3) and (6.95,1.4) .. (10.93,3.29)   ;
\draw  [dash pattern={on 0.84pt off 2.51pt}]  (210,210) -- (440,210) ;
\draw    (200,170) .. controls (259.4,170.99) and (387.41,211.18) .. (448.18,210.05) ;
\draw [shift={(450,210)}, rotate = 538.0899999999999] [color={rgb, 255:red, 0; green, 0; blue, 0 }  ][line width=0.75]    (10.93,-3.29) .. controls (6.95,-1.4) and (3.31,-0.3) .. (0,0) .. controls (3.31,0.3) and (6.95,1.4) .. (10.93,3.29)   ;
\draw    (580,200) -- (580,182) ;
\draw [shift={(580,180)}, rotate = 450] [color={rgb, 255:red, 0; green, 0; blue, 0 }  ][line width=0.75]    (10.93,-3.29) .. controls (6.95,-1.4) and (3.31,-0.3) .. (0,0) .. controls (3.31,0.3) and (6.95,1.4) .. (10.93,3.29)   ;
\draw    (580,160) -- (580,142) ;
\draw [shift={(580,140)}, rotate = 450] [color={rgb, 255:red, 0; green, 0; blue, 0 }  ][line width=0.75]    (10.93,-3.29) .. controls (6.95,-1.4) and (3.31,-0.3) .. (0,0) .. controls (3.31,0.3) and (6.95,1.4) .. (10.93,3.29)   ;
\draw    (630,130) -- (612,130) ;
\draw [shift={(610,130)}, rotate = 360] [color={rgb, 255:red, 0; green, 0; blue, 0 }  ][line width=0.75]    (10.93,-3.29) .. controls (6.95,-1.4) and (3.31,-0.3) .. (0,0) .. controls (3.31,0.3) and (6.95,1.4) .. (10.93,3.29)   ;
\draw    (509.93,142.19) .. controls (509.04,170) and (510.53,170) .. (570,170) ;
\draw [shift={(510,140)}, rotate = 91.91] [color={rgb, 255:red, 0; green, 0; blue, 0 }  ][line width=0.75]    (10.93,-3.29) .. controls (6.95,-1.4) and (3.31,-0.3) .. (0,0) .. controls (3.31,0.3) and (6.95,1.4) .. (10.93,3.29)   ;
\draw    (580,80) -- (580,50) ;
\draw [shift={(580,50)}, rotate = 270] [color={rgb, 255:red, 0; green, 0; blue, 0 }  ][fill={rgb, 255:red, 0; green, 0; blue, 0 }  ][line width=0.75]      (0, 0) circle [x radius= 3.35, y radius= 3.35]   ;
\draw    (580,120) -- (580,102) ;
\draw [shift={(580,100)}, rotate = 450] [color={rgb, 255:red, 0; green, 0; blue, 0 }  ][line width=0.75]    (10.93,-3.29) .. controls (6.95,-1.4) and (3.31,-0.3) .. (0,0) .. controls (3.31,0.3) and (6.95,1.4) .. (10.93,3.29)   ;
\draw    (200,90) .. controls (259.4,90.99) and (348.2,131.18) .. (408.19,130.05) ;
\draw [shift={(410,130)}, rotate = 538.0899999999999] [color={rgb, 255:red, 0; green, 0; blue, 0 }  ][line width=0.75]    (10.93,-3.29) .. controls (6.95,-1.4) and (3.31,-0.3) .. (0,0) .. controls (3.31,0.3) and (6.95,1.4) .. (10.93,3.29)   ;
\draw    (10,88.5) -- (160,88.5)(10,91.5) -- (160,91.5) ;
\draw    (270,88.5) -- (560,88.5)(270,91.5) -- (560,91.5) ;
\draw    (600,88.5) -- (660,88.5)(600,91.5) -- (660,91.5) ;
\draw    (270,168.5) -- (500,168.5)(270,171.5) -- (500,171.5) ;
\draw    (10,168.5) -- (160,168.5)(10,171.5) -- (160,171.5) ;
\draw    (10,248.5) -- (160,248.5)(10,251.5) -- (160,251.5) ;
\draw    (10,328.5) -- (160,328.5)(10,331.5) -- (160,331.5) ;
\draw    (10,408.5) -- (160,408.5)(10,411.5) -- (160,411.5) ;
\draw    (600,168.5) -- (660,168.5)(600,171.5) -- (660,171.5) ;
\draw    (600,248.5) -- (660,248.5)(600,251.5) -- (660,251.5) ;
\draw    (600,328.5) -- (660,328.5)(600,331.5) -- (660,331.5) ;
\draw    (600,408.5) -- (660,408.5)(600,411.5) -- (660,411.5) ;
\draw    (270,248.5) -- (560,248.5)(270,251.5) -- (560,251.5) ;
\draw    (270,328.5) -- (500,328.5)(270,331.5) -- (500,331.5) ;
\draw    (200,408.5) -- (560,408.5)(200,411.5) -- (560,411.5) ;
\draw    (180,420) -- (180,518) ;
\draw [shift={(180,520)}, rotate = 270] [color={rgb, 255:red, 0; green, 0; blue, 0 }  ][line width=0.75]    (10.93,-3.29) .. controls (6.95,-1.4) and (3.31,-0.3) .. (0,0) .. controls (3.31,0.3) and (6.95,1.4) .. (10.93,3.29)   ;
\draw    (10,488.5) -- (160,488.5)(10,491.5) -- (160,491.5) ;
\draw    (200,488.5) -- (250,488.5)(200,491.5) -- (250,491.5) ;
\draw    (290,488.5) -- (660,488.5)(290,491.5) -- (660,491.5) ;
\draw    (270,520) .. controls (269.5,473.75) and (269.83,480.67) .. (330,480) ;
\draw    (180,420) .. controls (177.17,470.67) and (519.17,398.67) .. (520,440) ;
\draw    (520,440) -- (520,448) ;
\draw [shift={(520,450)}, rotate = 270] [color={rgb, 255:red, 0; green, 0; blue, 0 }  ][line width=0.75]    (10.93,-3.29) .. controls (6.95,-1.4) and (3.31,-0.3) .. (0,0) .. controls (3.31,0.3) and (6.95,1.4) .. (10.93,3.29)   ;
\draw  [dash pattern={on 0.84pt off 2.51pt}]  (360,460) -- (410,460) ;
\draw    (497.97,471.21) .. controls (464.62,490.4) and (418.04,483.92) .. (330,480) ;
\draw [shift={(500,470)}, rotate = 148.33] [color={rgb, 255:red, 0; green, 0; blue, 0 }  ][line width=0.75]    (10.93,-3.29) .. controls (6.95,-1.4) and (3.31,-0.3) .. (0,0) .. controls (3.31,0.3) and (6.95,1.4) .. (10.93,3.29)   ;

\draw (182,49) node   [align=left] {$\displaystyle \mathbf{X}$};
\draw (114,131) node    {$\mathbf{W}_{1}$};
\draw (180.5,89) node    {$\mathbf{Z}_{0}$};
\draw (178.5,131) node    {$\mathbf{W}_{1}\mathbf{Z}_{0}$};
\draw (183,209) node    {$\varphi _{1}(\mathbf{Z}_{1})$};
\draw (180.5,251) node    {$\mathbf{Z}_{2}$};
\draw (178.5,289) node    {$\mathbf{W}_{2}\mathbf{Z}_{2}$};
\draw (114,291) node    {$\mathbf{W}_{2}$};
\draw (156,360.4) node [anchor=north west][inner sep=0.75pt]    {$\varphi _{2}(\mathbf{Z}_{3})$};
\draw (180.5,409) node    {$\mathbf{Z}_{4}$};
\draw (284,457) node    {$\mathcal{L} =\frac{1}{2m} \ \mathbf{1}^{T} \cdot (\mathbf{Y} -\mathbf{Z}_{4})^{\circ 2}$};
\draw (487.5,456) node    {$\nabla \mathcal{L} =-\frac{1}{m}(\mathbf{Y} -\mathbf{Z}_{4})$};
\draw (581.5,410) node    {$\mathbf{D}_{4}$};
\draw (181.5,169) node    {$\mathbf{Z}_{1}$};
\draw (181.5,329) node    {$\mathbf{Z}_{3}$};
\draw (481.5,368.5) node    {$\varphi ^{'}_{2}(\mathbf{Z}_{3})$};
\draw (583,369) node    {$\mathbf{\varphi ^{'}_{2}( Z_{3}) \circ D}_{4}$};
\draw (582.5,331) node    {$\mathbf{D}_{3}$};
\draw (576.5,291) node    {$\mathbf{W}^{T}_{2}\mathbf{D}_{3}$};
\draw (582.5,251) node    {$\mathbf{D}_{2}$};
\draw (481.5,208.5) node    {$\varphi ^{'}_{1}(\mathbf{Z}_{1})$};
\draw (584,207) node    {$\mathbf{\varphi ^{'}_{1}( Z_{1}) \circ D}_{2}$};
\draw (654,291) node    {$\mathbf{W}_{2}$};
\draw (481.5,290.5) node    {$\Delta \mathbf{W}_{2} =-\eta \mathbf{D}_{3}\mathbf{Z}^{T}_{2}$};
\draw (582.5,171) node    {$\mathbf{D}_{1}$};
\draw (576.5,131) node    {$\mathbf{W}^{T}_{1}\mathbf{D}_{1}$};
\draw (654,131) node    {$\mathbf{W}_{1}$};
\draw (481,129.5) node    {$\Delta \mathbf{W}_{1} =-\eta \mathbf{D}_{1}\mathbf{Z}^{T}_{0}$};
\draw (582.5,91) node    {$\mathbf{D}_{0}$};
\draw (11,132) node [anchor=west] [inner sep=0.75pt]  [font=\small] [align=left] {LAYER 1\\(FC)};
\draw (11,212) node [anchor=west] [inner sep=0.75pt]  [font=\small] [align=left] {LAYER 2\\(ACTIVATION)};
\draw (11,292) node [anchor=west] [inner sep=0.75pt]  [font=\small] [align=left] {LAYER 3\\(FC)};
\draw (11,372) node [anchor=west] [inner sep=0.75pt]  [font=\small] [align=left] {LAYER 4\\(ACTIVATION)};
\draw (11,56) node [anchor=west] [inner sep=0.75pt]  [font=\small] [align=left] {INPUT};
\draw (11,451) node [anchor=west] [inner sep=0.75pt]  [font=\small] [align=left] {LOSS};
\draw (181.5,533) node   [align=left] {$\displaystyle \mathring{\mathbf{Y}}$};
\draw (9,535) node [anchor=west] [inner sep=0.75pt]  [font=\small] [align=left] {OUTPUT};
\draw (268.5,536) node   [align=left] {$\displaystyle \mathbf{Y}$};
\draw (135,12) node [anchor=north west][inner sep=0.75pt]   [align=left] {FORWARD};
\draw (575,20.5) node   [align=left] {BACKWARD};

\end{tikzpicture}
}
\caption{Training depiction of a neural network. Here, \( \mathring{\bm{Y}} \) represents the network's output, \( \bm{Y} \) denotes the target, and \( \bm{W}_i \), \( \varphi_i(\cdot) \), \( \bm{Z}_i \) respectively symbolize the weighting matrix, the activation function, and the \( i \)-th layer's output. \( \bm{D}_i \) stands for the delta factor, crucial for weight updates at the \( i \)-th layer.}\label{fig:mlp-process} 
\end{figure}

It is important to recognize that while the above generalizations are applicable across many neural network architectures, specific instances or novel architectures may require additional analysis. Hence, careful validation remains advisable when extending these principles to unconventional or emerging neural models.

Our findings, grounded in the fundamental mechanics elaborated in the theoretical propositions, align with and complement a wide spectrum of neural network optimization methodologies reviewed in Section~2. These methods, which span interventions on activations, weights, and gradients, uniformly exhibit a shared characteristic: they predominantly modulate the magnitude rather than altering the inherent direction of gradient-based updates.

\paragraph{Weights and Gradients} The relationship between weight updates and gradients is explicit and intuitive. Most existing optimization techniques affecting weights or gradients—including regularization, pruning, gradient clipping, and momentum-based approaches—primarily influence the magnitude of updates. Conversely, the direction, dictated by the gradient of the loss function relative to the weights, remains fundamentally preserved.

\paragraph{Activations} The role of activations warrants deeper examination. Techniques such as batch normalization, layer normalization, instance normalization, and related normalization approaches explicitly adjust the scale and statistical distribution of activations. Crucially, these adjustments maintain the relative ordering and proportionality among activation components. This explicit preservation ensures that the directionality of weight updates—determined fundamentally by preceding activations—remains invariant. Even activation functions that introduce strong non-linear transformations (e.g., ReLU, sigmoid, tanh) explicitly affect only the magnitude through their derivatives, leaving the direction of the weight update consistently defined by the pre-activation vector.

This definitive understanding that the gradient's direction is independent of, and more significant than, its magnitude naturally facilitates the decoupling of forward activations and backward gradient paths. Since the critical directional information is determined solely by linear neuron connections and activations from preceding layers, it becomes feasible—and theoretically justified—to separate the gradient magnitude from its traditional dependency on forward activation derivatives. This insight paves the way for novel gradient manipulation strategies, allowing greater flexibility and potentially improving neural network training methodologies.

These theoretical insights definitively establish that the invariance and primacy of gradient directions enable the decoupling of forward activations from backward gradient computations. Since directional information is governed exclusively by linear neuron connections and preceding layer activations, it becomes both feasible and theoretically sound to dissociate gradient magnitudes from their conventional reliance on activation derivatives. Consequently, gradient-based corrections can now be explicitly modulated for targeted purposes—such as accelerating convergence, enhancing generalization, or mitigating overfitting—while explicitly maintaining the intrinsic directional pathway inherent in neural learning dynamics. Recognizing and leveraging this fundamental property opens novel avenues in neural network training, enabling innovative magnitude adjustments that explicitly optimize learning performance without compromising directional consistency.

The next section explicitly presents empirical validations that corroborate these theoretical propositions, further illuminating essential nuances of gradient-based learning dynamics and providing concrete evidence of the efficacy of forward-backward decoupling in neural networks.

\section{Experimental Validation}

To empirically validate the theoretical insights presented in Section~3, we conducted a series of experiments aimed at both confirming our conclusions and uncovering additional insights into the decoupling of forward activations and backward gradients.

\subsection{Architectures}

We employed three neural network architectures: the Single Unit Classifier (SUC), the Multi-Layer Perceptron (MLP), and the Convolutional Neural Network (LeNet-5). These architectures were intentionally chosen for their simplicity, enabling a clear examination of fundamental concepts and isolating the specific effects of modifying gradient computations.

Specifically, the SUC allowed us to directly examine the impact on the most fundamental neural unit. The MLP provided a straightforward yet sufficiently complex setup to transparently analyze the mathematical nuances of backpropagation under relaxed traditional constraints. In contrast, LeNet-5 introduced greater complexity through its convolutional layers and quadratic-mean pooling operations, better representing realistic deep learning scenarios. Collectively, these models facilitated a comprehensive exploration of how decoupling forward activation functions from backward gradient computations influences various architectural elements, including fully connected layers, convolutional operations, and pooling mechanisms.

We deliberately omitted more advanced architectures, such as deeper convolutional networks or Transformers, because their complexity might obscure the essential theoretical phenomena under investigation and hinder the precise identification of the effects of our modifications.

{
\subsection{Datasets}

For our experiments, we employed four benchmark datasets, each selected for its relevance to different architectures and experimental analyses:

\begin{itemize}

\item \textbf{Ones Dataset (Synthetic)}:  
We designed a synthetic binary classification task to determine whether the proportion of '1's within an input vector \(\bm{X}\) exceeds a specified threshold (e.g., 70\%). This task produces a linearly separable decision boundary, clearly delineating distinct, non-overlapping convex regions. The simplicity of this dataset enables precise investigation into how modifications in gradient computations of activation functions influence parameter updates, isolating this effect from potential confounding factors.

\item \textbf{Digits Dataset}~\cite{digits}:  
The Digits dataset consists of 1,797 samples distributed across 10 classes representing handwritten digits (0--9), with class sizes ranging between 174 and 183 samples, ensuring balanced representation. Each digit is represented by an 8×8 pixel image, converted into a 64-dimensional numerical feature vector with pixel values ranging from 0 to 16. This dataset, sourced from the scikit-learn library, is frequently employed for benchmarking numerical classification algorithms and evaluating lower-dimensional representations.

\item \textbf{MNIST Dataset}~\cite{mnist}:  
The MNIST dataset is a widely recognized benchmark comprising 70,000 grayscale images of handwritten digits (0--9), divided into 60,000 training samples and 10,000 test samples. Each image measures \(28\times 28\) pixels, yielding a 784-dimensional feature vector per sample. Pixel intensity values range from 0 (black) to 255 (white). Due to its structured nature and standardized labeling, MNIST serves as a fundamental benchmark for convolutional neural networks and classification algorithms.

\item \textbf{Fashion-MNIST Dataset}~\cite{fashion-mnist}:  
Fashion-MNIST, structurally similar to MNIST, contains 70,000 grayscale images evenly divided into 10 categories representing various clothing items (e.g., t-shirts, trousers, dresses). The dataset is partitioned into 60,000 training and 10,000 test samples, each comprising a \(28\times 28\) pixel image converted into a 784-dimensional feature vector. Pixel intensity values range from 0 to 255. Fashion-MNIST presents a more challenging scenario than MNIST, due to increased complexity and visual similarities among the clothing categories, making it suitable for testing algorithm robustness and generalization capabilities.

\end{itemize}

\subsection{Detailed Experimental Setup}

To ensure statistical robustness while balancing computational feasibility, each experiment was independently executed 20 times. Datasets were consistently partitioned into training and test sets using an 80/20 split, a standard practice that provides representative test data for reliable performance evaluation.

We explored two distinct experimental configurations, termed \textit{tied} and \textit{untied}. The \textit{tied} configuration, denoted as ``Forward Function / Backward Gradient'' (\(f/f^\prime\)), employs the gradient derived directly from the forward activation function. In contrast, the \textit{untied} configuration (\(f/g^\prime\)) permits an independent, explicitly defined gradient function separate from the forward activation. Notably, the forward activation functions remain identical in both configurations; the sole distinction lies in the backward gradient computations, which, in the untied case, are user-defined and not constrained to match the derivative of the activation functions.

Examples of the gradient functions used and alternative backward gradient computations are shown in Figures~\ref{fig:gradcontributions} and~\ref{fig:jamming}, respectively. To explicitly isolate and analyze the impact of decoupling gradients from forward activation functions, we intentionally avoided extensive hyperparameter tuning through validation sets. All models and experiments were implemented using TensorFlow and PyTorch, and executed on GPU-accelerated computing platforms to ensure computational efficiency.

\begin{figure}[ht]
  \centering
   \subfloat[Constant]{\includegraphics[scale=0.25]{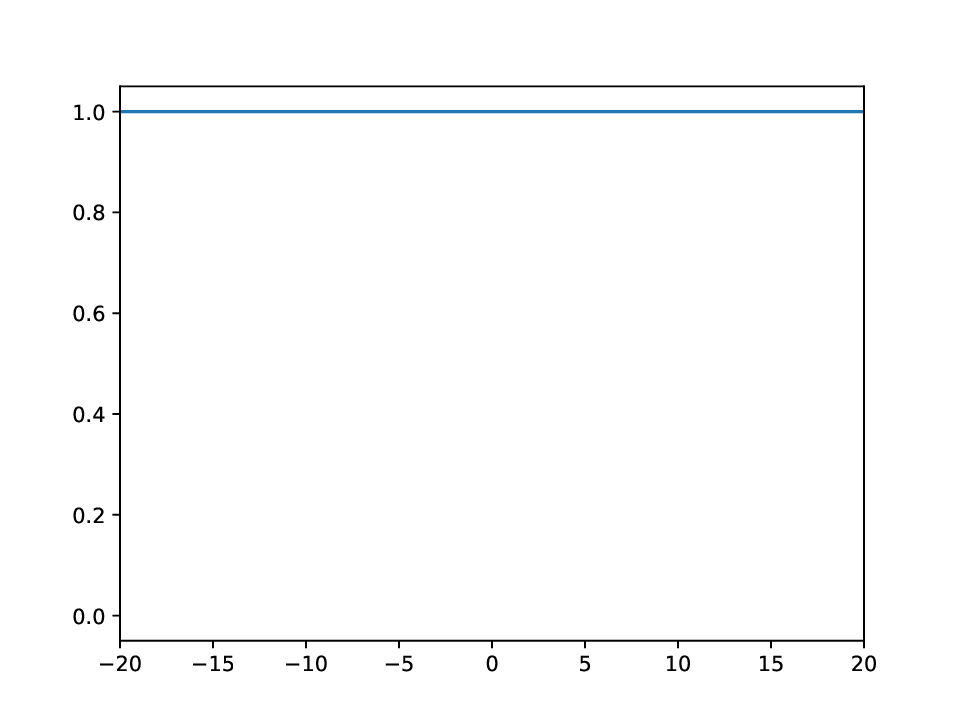}
    \label{fig:const}}
  \hfill
  \subfloat[Rectangular]{\includegraphics[scale=0.25]{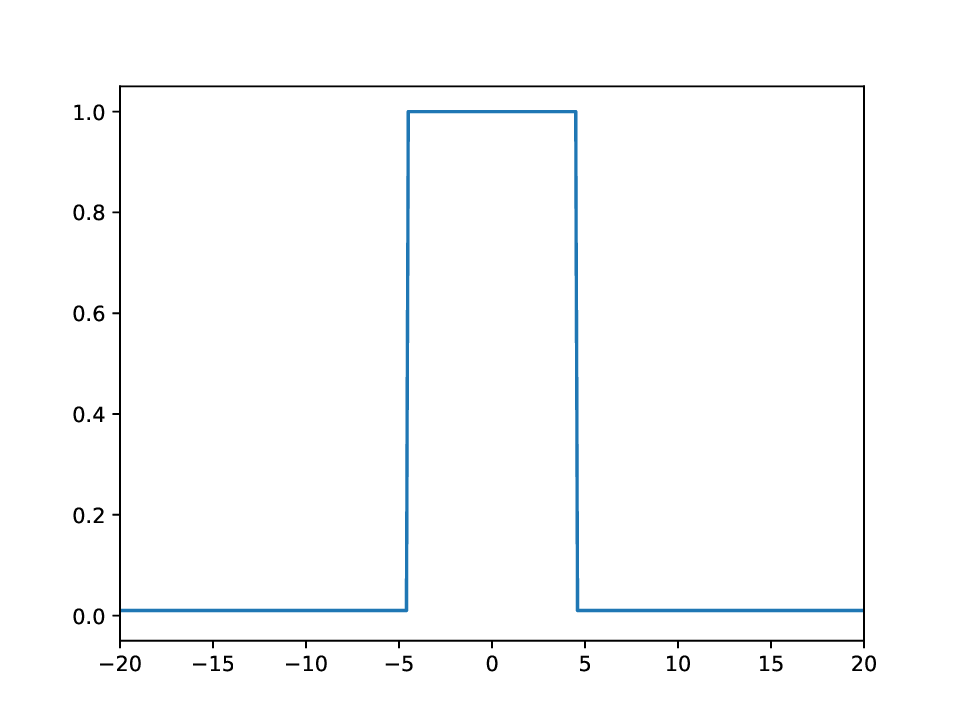}
    \label{fig:rect}}
     \hfill
  \subfloat[Triangular]{\includegraphics[scale=0.25]{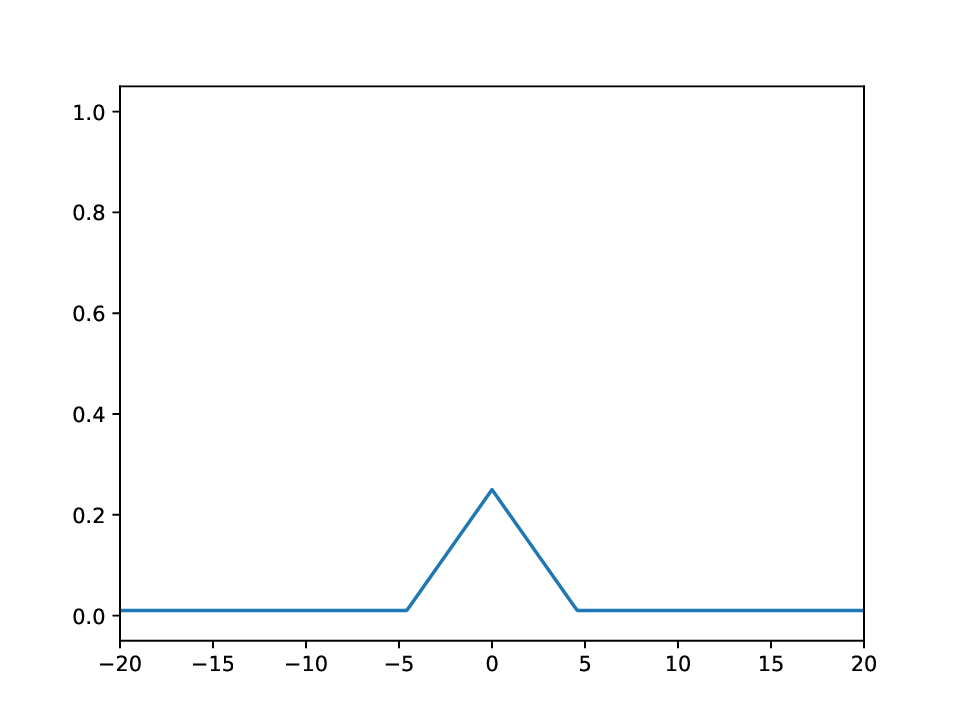}
    \label{fig:triangular}}
  
  \caption{Gradient contributions used in the experiments in place of the activation functions gradients. Specifically, we have considered a constant function (a), a rectangular function (c) a triangular function(d).}
  \label{fig:gradcontributions}
\end{figure}

We chose accuracy as the primary evaluation metric due to its straightforward interpretability and suitability for datasets with balanced class distributions, such as those employed here (Ones, Digits, MNIST, and Fashion-MNIST). Accuracy was calculated on the separate test sets (20\% of total samples), with results averaged across 20 independent experimental runs. We report both the mean and standard deviation to indicate consistency and robustness of performance. Although alternative metrics (e.g., precision, recall, F1-score, AUC-ROC) were considered, accuracy was preferred for its simplicity and ease of comparability, particularly given the balanced nature of the datasets and the absence of significant class imbalance.

}

\subsection{Experiment 1: Single Unit Classifier}

In the initial experiment, we examine a basic neural architecture consisting of a single-layer unit classifier (SUC), as depicted in Fig. \ref{fig:SUC}. Historically significant as a foundational model in neural network research from the mid-1940s to the mid-1960s, the SUC serves as an ideal platform for verifying the theoretical propositions outlined previously.

The classification task chosen for this experiment involves determining whether the proportion of 1s in the input vector \( \bm{X} \) exceeds a predefined threshold (e.g., 70\%). Theoretically, the decision space for this problem is linearly separable, ensuring non-intersecting convex envelopes for each feature class. This setup enables a focused investigation into the role of the activation function's gradient during backpropagation and its impact on parameter updates.

Using a constant gradient in place of the logistic function's derivative during backpropagation presents several notable benefits:
\begin{enumerate}
\item \textbf{Computational and Memory Efficiency:} Eliminating the derivative calculation simplifies and speeds up backpropagation, reducing computational overhead and memory usage.
\item \textbf{Mitigation of Saturation Issues:} Avoiding conventional activation gradients prevents the slowdown or stagnation associated with activation saturation, common in sigmoid or tanh functions.
\item \textbf{Enhanced Training Stability:} Consistent gradient magnitude ensures stable and predictable training, particularly useful when traditional activation gradients become excessively small or large.
\end{enumerate}

Table \ref{tab:exp1} summarizes results obtained using identity, sigmoid, and Heaviside step activation functions, evaluated across varying input dimensions (2, 10, 100), batch sizes (32, 64), and learning rates (0.01, 0.1).

\begin{table}[!ht]
\setlength\tabcolsep{0pt}
\caption{Experiment 1 results. Training a Single Unit Classifier under tied and untied configurations across varying dimensions (Dim), batch sizes (BS), and learning rates (LR). The untied configuration employs a constant gradient instead of the logistic function's derivative.}
\label{tab:exp1}
\begin{tabular*}{\textwidth}{@{\extracolsep{\fill}} ccc|cc}
\toprule
\multicolumn{3}{c}{\textbf{Configuration}} & Log/d(Log) & \textit{Log/1} \\
\multicolumn{1}{c}{\textbf{Dim}} & \multicolumn{1}{c}{\textbf{BS}} & \multicolumn{1}{c}{\textbf{LR}} & \multicolumn{2}{c}{}\\
\midrule
2 & 32 & 0.01 & 0.9938 (0.003) & 0.9967 (0.002) \\
 & & 0.1 & \textbf{0.9982 (0.001)} & \textbf{0.9986 (0.001)} \\
 & 64 & 0.01 & 0.9906 (0.003) & 0.9950 (0.002) \\
 & & 0.1 & 0.9976 (0.002) & 0.9984 (0.001) \\
\hline
10 & 32 & 0.01 & 0.9957 (0.002) & 0.9965 (0.002) \\
 & & 0.1 & 0.9971 (0.001) & 0.9977 (0.001) \\
 & 64 & 0.01 & 0.9948 (0.002) & 0.9961 (0.002) \\
 & & 0.1 & 0.9967 (0.001) & 0.9973 (0.001) \\
\hline
100 & 32 & 0.01 & 0.9853 (0.004) & 0.9871 (0.003) \\
 & & 0.1 & 0.9886 (0.002) & 0.9903 (0.002) \\
 & 64 & 0.01 & 0.9536 (0.009) & 0.9859 (0.003) \\
 & & 0.1 & 0.9880 (0.002) & 0.9893 (0.002) \\
\bottomrule
\end{tabular*}
\end{table}

Results from Table \ref{tab:exp1} indicate that replacing the logistic activation gradient with a constant value does not negatively affect accuracy; rather, it demonstrates comparable or improved robustness, especially at higher learning rates. This finding challenges the conventional belief that symmetric gradient propagation (tied configuration) is critical for effective learning.

The untied configuration consistently shows stable or superior performance relative to the tied configuration, demonstrating robustness across varying input dimensions, batch sizes, and learning rates. As expected, accuracy declines slightly with increased dimensionality for both, particularly at higher learning rates, performance remains equal to or better than the traditional tied method. The minimal standard deviations highlight the stability and reliability of the untied approach across repeated trials.

Overall, Experiment 1 presents the untied configuration as a viable and robust alternative to traditional backpropagation methods, aligning well with our theoretical propositions.

\subsection{Experiment 2: Multilayer Perceptron}

In this experiment, conducted using the Digits dataset, we explored the training effectiveness of a standard two-layer Multi-Layer Perceptron (MLP) using sigmoid activation when the activation function's gradient is replaced by a constant. We considered two configurations: the \textit{tied configuration}, where the gradient is derived from the sigmoid derivative, and the \textit{untied configuration}, where the sigmoid gradient is substituted by a constant value of 1. Experiments spanned batch sizes (BS) of 32, 64, and 128, and learning rates (LR) of 0.01, 0.05, and 0.1.

\begin{table}[!ht]
\setlength\tabcolsep{0pt}
\caption{Experiment 2 results. Mean accuracy (with standard deviations) on the test set for MLPs under different hidden-unit configurations, batch sizes (BS), and learning rates (LR). \textit{Log/d(Log)} represents the tied configuration; \textit{Log/1} denotes the untied configuration.}
\label{exp2}
{\footnotesize
\begin{tabular*}{\textwidth}{@{\extracolsep{\fill}} cc|cccc}
\toprule
\multicolumn{2}{c}{\textbf{Hidden Units}} & \multicolumn{2}{c}{32} & \multicolumn{2}{c}{64} \\
\multicolumn{2}{c}{\textbf{Configuration}} & Log/d(Log) & \textit{Log/1} & Log/d(Log) & \textit{Log/1}\\
\textbf{BS} & \textbf{LR} & & & & \\
\midrule
32 & 0.01 & 0.9663 (0.007) & 0.9351 (0.013) & 0.9733 (0.008) & 0.9474 (0.009) \\
 & 0.05 & \textbf{0.9666 (0.010)} & 0.9364 (0.011) & 0.9732 (0.007) & 0.9452 (0.011) \\
 & 0.1 & 0.9585 (0.017) & 0.9364 (0.015) & 0.9713 (0.009) & 0.9449 (0.011) \\
\hline
64 & 0.01 & 0.9597 (0.008) & \textbf{0.9375 (0.013)} & \textbf{0.9742 (0.009)} & \textbf{0.9480 (0.013)} \\
 & 0.05 & 0.9616 (0.013) & 0.9364 (0.012) & 0.9722 (0.007) & 0.9506 (0.009) \\
 & 0.1 & 0.8502 (0.064) & 0.9339 (0.013) & 0.5841 (0.344) & 0.9411 (0.019) \\
\hline
128 & 0.01 & 0.9561 (0.014) & 0.9217 (0.032) & 0.9727 (0.009) & 0.9412 (0.012) \\
 & 0.05 & 0.8119 (0.131) & 0.9326 (0.015) & 0.7725 (0.209) & 0.9484 (0.017) \\
 & 0.1 & 0.0984 (0.025) & 0.8732 (0.041) & 0.0947 (0.016) & 0.8732 (0.037) \\
\bottomrule
\end{tabular*}
}
\end{table}

\begin{figure}[htp] 
    \centering
    \subfloat[Training accuracy]{%
        \includegraphics[width=0.5\textwidth]{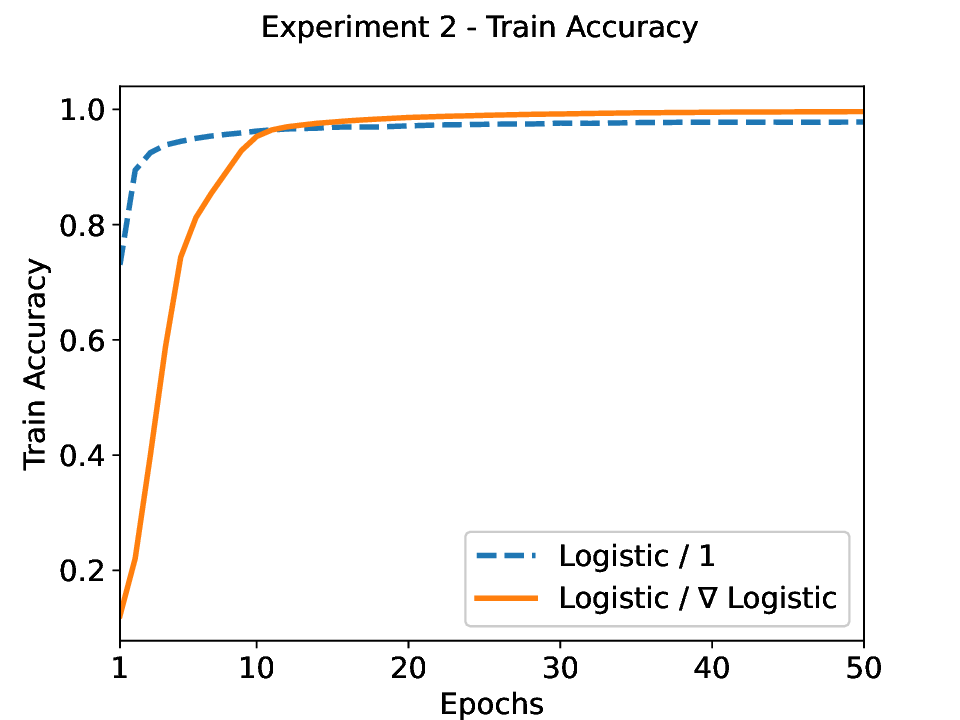}%
        \label{fig:a}%
    }%
    \hfill%
    \subfloat[Training loss]{%
        \includegraphics[width=0.5\textwidth]{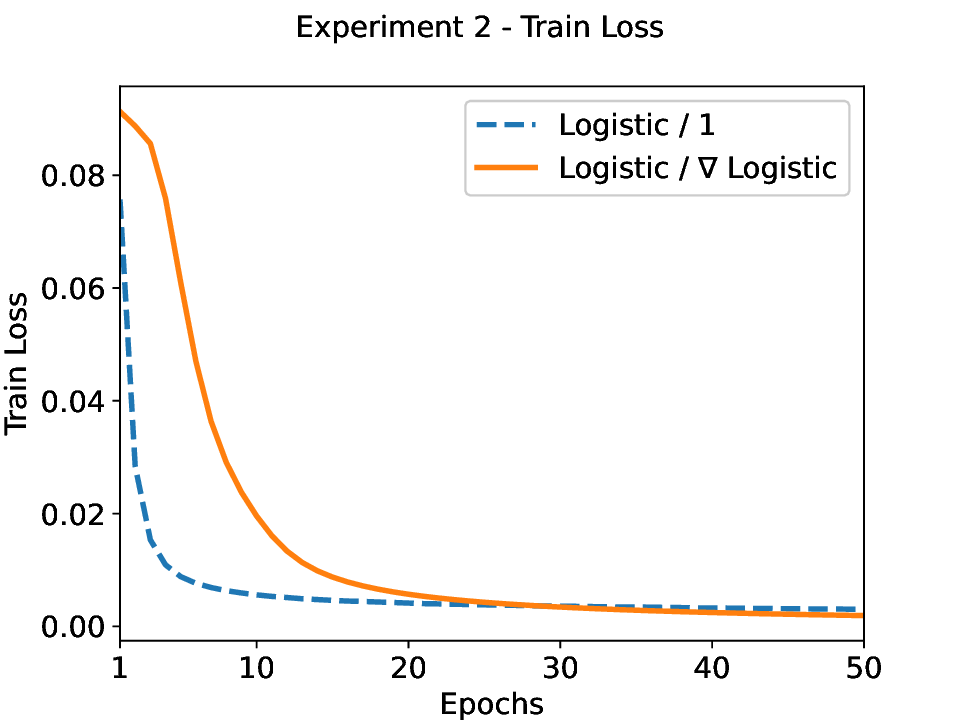}%
        \label{fig:b}%
    }%
    \caption{Training accuracy and loss trends for tied and untied gradient configurations in the MLP architecture with sigmoid activation (Experiment 2).}
\end{figure}

Table \ref{exp2} and Figures \ref{fig:a}-\ref{fig:b} highlight several findings from Experiment 2:
\begin{enumerate}

\item \textbf{Convergence capability}: The untied configuration successfully converges during training, although the tied configuration generally achieves slightly higher accuracy across most experimental conditions.

\item \textbf{Rapid convergence of the untied configuration}: Figures \ref{fig:a} and \ref{fig:b} demonstrate that the untied configuration achieves faster initial convergence compared to the tied configuration, indicating potential efficiency advantages.

\item \textbf{Sensitivity to learning rate}: Higher learning rates significantly impact the tied configuration, causing instability. This effect is particularly evident for larger hidden layers (128 units), where the tied configuration fails to converge at a learning rate (LR) of 0.1, whereas the untied configuration continues to converge, albeit with reduced performance.

\item \textbf{Stability of the untied configuration}: Using a constant gradient provides inherent stability, consistently reflected in robust performance across varying conditions.

\end{enumerate}

These findings support the use of constant gradients as a practical alternative, particularly when stability and training speed are critical considerations. This configuration simplifies gradient computation and offers significant potential to reduce training time, as evidenced by the accelerated convergence shown in Figure \ref{fig:a}. Thus, employing constant gradient modulation not only maintains effective training convergence but also enhances computational efficiency and stability.
 
\subsection{Experiment 3: Convolutional Neural Networks}

In this experiment, we extend our investigation to convolutional neural networks (CNNs), aiming to explore the implications of constant and alternative gradient modulation methods. The well-known MNIST dataset, a widely-used benchmark for CNN evaluation, was utilized alongside the canonical LeNet5 architecture to assess three distinct gradient modulation scenarios. First, the tied configuration (Logistic/d(Logistic)) directly uses the gradient derived from the logistic activation function's derivative. Second, the untied constant gradient configuration (Logistic/1) replaces this traditional gradient with a constant value of 1, thereby simplifying the gradient computation process. Third, the untied rectangular gradient configuration (Logistic/Rectangular) employs a rectangular window function to modulate gradients, introducing a controlled variation in gradient magnitude independent of the activation function's derivative.

Experiment 3 explored the effects of various gradient modulation strategies—tied logistic derivatives, constant gradients, and rectangular gradients—in training the LeNet5 convolutional neural network (CNN). A variety of hyperparameter settings were tested, specifically batch sizes of 256, 512, and 1024, paired with learning rates of 0.005, 0.01, 0.05, and 0.1. The mean accuracy and standard deviation across 20 experimental runs are detailed in Table \ref{exp3}.
\begin{table}[ht]
\setlength\tabcolsep{0pt}
\caption{Experiment 3 results. Mean test accuracy (with standard deviations) for LeNet5 under tied (\textit{Log / d(Log)}), constant gradient (\textit{Log / 1}), and rectangular gradient (\textit{Log / Rect}) configurations. Results are provided for varying batch sizes (BS) and learning rates (LR).}
\label{exp3}
\begin{tabular*}{\textwidth}{@{\extracolsep{\fill}} cc|ccc}
\toprule
 \textbf{BS} & \textbf{LR} & Log/d(Log) & \textit{Log/1} & \textit{Log/Rect} \\
\midrule
256 & 0.005 & 0.1118 (0.002) & 0.7783 (0.155) & 0.9601 (0.002) \\
    & 0.01  & 0.1170 (0.016) & \textbf{0.9021 (0.010)} & 0.9771 (0.001) \\
    & 0.05  & 0.9708 (0.001) & 0.8852 (0.017) & 0.9827 (0.002) \\
    & 0.1   & \textbf{0.9812 (0.001)} & 0.6708 (0.289) & \textbf{0.9831 (0.002)} \\
\hline
512 & 0.005 & 0.1118 (0.002) & 0.1989 (0.166) & 0.7232 (0.002) \\
    & 0.01  & 0.1118 (0.002) & 0.8666 (0.004) & 0.9596 (0.002) \\
    & 0.05  & 0.9236 (0.002) & 0.7175 (0.303) & 0.9819 (0.002) \\
    & 0.1   & 0.9697 (0.001) & 0.8612 (0.011) & 0.9801 (0.001) \\
\hline
1024& 0.005 & 0.1119 (0.002) & 0.1751 (0.065) & 0.1119 (0.002) \\
    & 0.01  & 0.1119 (0.002) & 0.7698 (0.007) & 0.7155 (0.002) \\
    & 0.05  & 0.1865 (0.046) & 0.8182 (0.163) & 0.9784 (0.002) \\
    & 0.1   & 0.9104 (0.002) & 0.8560 (0.010) & 0.9655 (0.030) \\
\bottomrule
\end{tabular*}
\end{table}

All configurations displayed considerable sensitivity to the chosen learning rate, which is not surprising. Performance varied significantly at higher learning rates, emphasizing the importance of meticulous hyperparameter tuning. Furthermore, larger batch sizes combined with lower learning rates presented convergence challenges, highlighting the critical role that careful batch size selection plays in ensuring effective model training.

A key outcome of this experiment is that the Log/Rect configuration consistently maintained stable performance across varying hyperparameters, unlike the tied configuration, which exhibited significant performance declines under certain conditions. In contrast, the Log/Const configuration showed a notably larger range of performance variability compared to the baseline tied configuration. This observation underscores the potential advantages of employing rectangular gradient modulation strategies, highlighting their robustness and stability in convolutional neural network training.

\subsection{Experiment 4: CNNs with ReLU}

Building upon insights from previous experiments, we further explore convolutional neural networks (CNNs) by examining the effects of different gradient modulation approaches in conjunction with Rectified Linear Unit (ReLU) activations. We maintained the LeNet5 architecture as our standard experimental baseline.

The analysis focused on comparing five specific gradient modulation configurations:
\begin{enumerate}
\item The traditional tied configuration (ReLU/d(ReLU)), utilizing the gradient of the ReLU activation function.
\item An untied configuration using the derivative of the logistic activation (ReLU/d(Logistic)).
\item An untied configuration employing a constant gradient (ReLU/1).
\item An untied rectangular gradient modulation approach (ReLU/Rect).
\item An untied triangular gradient modulation strategy (ReLU/Triang).
\end{enumerate}

A comprehensive set of hyperparameter combinations was assessed, specifically batch sizes of 256, 512, and 1024, paired with learning rates of 0.005, 0.01, 0.05, and 0.1. Mean test accuracies and standard deviations from 20 experimental runs are presented in Table \ref{exp4}.

\begin{table}[ht]
\setlength\tabcolsep{0pt}
\caption{Experiment 4 results. Mean test accuracy (standard deviations in parentheses) for LeNet5 with ReLU activations under one tied (\textit{ReLU/H}) and four untied configurations (\textit{ReLU/d(Logistic)}, \textit{ReLU/1}, \textit{ReLU/Triang}, \textit{ReLU/Rect}), evaluated across varying batch sizes (BS) and learning rates (LR). Bold values highlight best performance per row.}
\label{exp4}
{\scriptsize
\begin{tabular*}{\textwidth}{@{\extracolsep{\fill}}cc|ccccc}
\toprule
 \multicolumn{2}{c}{\textbf{Configuration}} & ReLU/H & \textit{ReLU/d(Logistic)} & \textit{ReLU/1} & \textit{ReLU/Triang} & \textit{ReLU/Rect} \\
 \multicolumn{1}{c}{\textbf{BS}} & \multicolumn{1}{c}{\textbf{LR}} \\
\midrule
256 & 0.005 & 0.9832 (0.001) & 0.9055 (0.002) & \textbf{0.9810 (0.001)} & 0.9022 (0.002) & 0.9735 (0.001) \\
 & 0.01 & 0.9857 (0.001) & 0.9426 (0.001) & 0.9789 (0.002) & 0.9409 (0.001) & 0.9777 (0.001) \\
 & 0.05 & \textbf{0.9883 (0.001)} & 0.9762 (0.001) & 0.1123 (0.002) & 0.9766 (0.001) & 0.9779 (0.001) \\
 & 0.1 & 0.9003 (0.262) & \textbf{0.9809 (0.001)} & 0.1120 (0.002) & \textbf{0.9816 (0.001)} & 0.2794 (0.332) \\
\hline
512 & 0.005 & 0.9785 (0.001) & 0.7785 (0.003) & 0.9787 (0.001) & 0.7469 (0.003) & 0.9640 (0.001) \\
 & 0.01 & 0.9833 (0.001) & 0.9055 (0.002) & 0.9809 (0.001) & 0.9022 (0.002) & 0.9735 (0.001) \\
 & 0.05 & 0.9875 (0.001) & 0.9684 (0.001) & 0.1111 (0.001) & 0.9683 (0.001) & 0.9798 (0.001) \\
 & 0.1 & 0.9630 (0.004) & 0.9761 (0.001) & 0.1125 (0.002) & 0.9766 (0.001) & 0.9768 (0.004) \\
\hline
1024 & 0.005 & 0.9679 (0.001) & 0.5075 (0.011) & 0.9727 (0.001) & 0.5020 (0.011) & 0.9464 (0.001) \\
 & 0.01 & 0.9783 (0.001) & 0.7737 (0.003) & 0.9785 (0.001) & 0.7405 (0.003) & 0.9641 (0.001) \\
 & 0.05 & 0.9866 (0.001) & 0.9511 (0.001) & 0.1119 (0.002) & 0.9497 (0.001) & 0.9788 (0.001) \\
 & 0.1 & 0.8986 (0.262) & 0.9683 (0.001) & 0.1111 (0.002) & 0.9682 (0.001) & \textbf{0.9804 (0.001)} \\
\bottomrule
\end{tabular*}}
\end{table}

The traditional tied configuration (ReLU/H) consistently delivered strong and reliable performance, affirming its established efficacy. However, several interesting observations emerged concerning the untied gradient configurations. Both the ReLU/d(Logistic) and ReLU/Rect configurations frequently matched or even exceeded the performance of the traditional tied approach under certain hyperparameter settings, highlighting their potential utility. Notably, larger batch sizes tended to enhance performance in untied configurations, possibly due to improved gradient estimation and stability in updates.

Conversely, the ReLU/1 configuration experienced notable performance drops under specific conditions, underscoring its sensitivity and potential instability relative to other configurations. Furthermore, pronounced variations in performance at higher learning rates across all configurations clearly illustrate the critical role of careful hyperparameter tuning.

In summary, these findings demonstrate the practical viability and potential advantages of employing untied gradient modulation configurations when training CNNs with ReLU activations. These approaches, particularly the ReLU/Rect configuration, offer promising alternatives to conventional methods. Nevertheless, achieving optimal performance demands meticulous hyperparameter optimization, given the observed sensitivity to learning rate adjustments.

\subsection{Experiment 5: Randomized Gradient Modulation}

In this experiment, we investigate neural network training when the gradient magnitude is entirely disrupted and replaced by uniform random noise sampled from the interval \([0,1]\). To ensure comparability, all noise functions are normalized such that their mean aligns with the average gradient magnitude observed in traditional gradient configurations. We use the LeNet5 architecture trained on the DIGITS dataset. Specifically, we explore three distinct noise modulation methods: \textit{Full-Jamming}, \textit{Positive-Jamming}, and \textit{Rectangular-Jamming}, as illustrated in Figure~\ref{fig:jamming}.

The noise modulation techniques are defined as follows:

\begin{itemize}
    \item \textbf{Full-Jamming (Fig.~\ref{fig:fulljam}):} The backward gradient is completely replaced by a uniformly random value from the interval \([0,1]\), representing the most disruptive scenario. This approach ensures that only the sign (directionality) of the gradient is preserved.
    
    \item \textbf{Positive-Jamming (Fig.~\ref{fig:posjam}):} The backward gradient is replaced by a uniformly random value from the interval \([0,1]\) exclusively when the feedback gradient is non-negative. Negative feedback gradient values are set to \(0\).
    
    \item \textbf{Rectangular-Jamming (Fig.~\ref{fig:rectjam}):} The backward gradient is replaced by a uniformly random value from the interval \([0,1]\) within a specified interval of feedback gradient values and is set to \(0\) outside this range.
\end{itemize}

\begin{figure}[!ht]
  \centering
  \subfloat[Full-Jamming]{\includegraphics[scale=0.30]{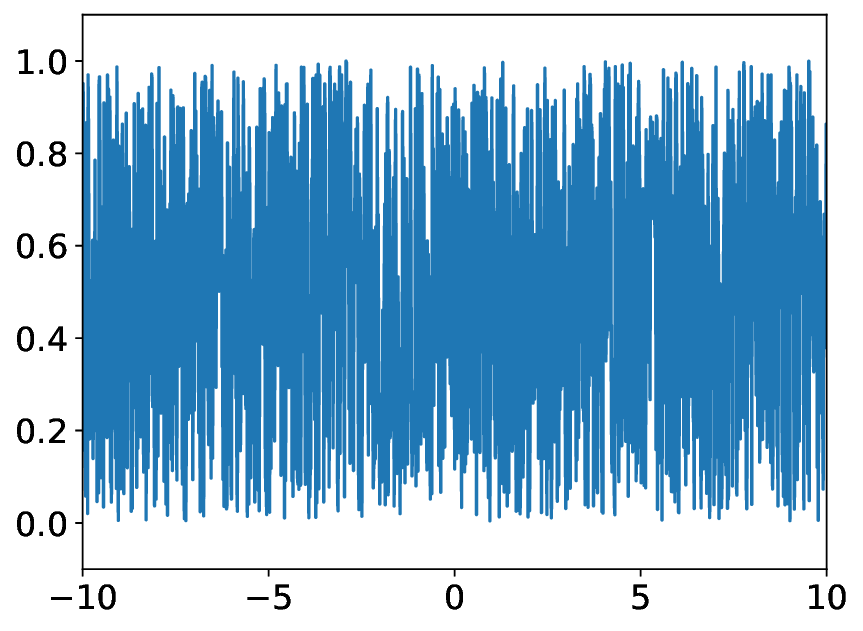}
    \label{fig:fulljam}}
  \hfill
  \subfloat[Positive-Jamming]{\includegraphics[scale=0.30]{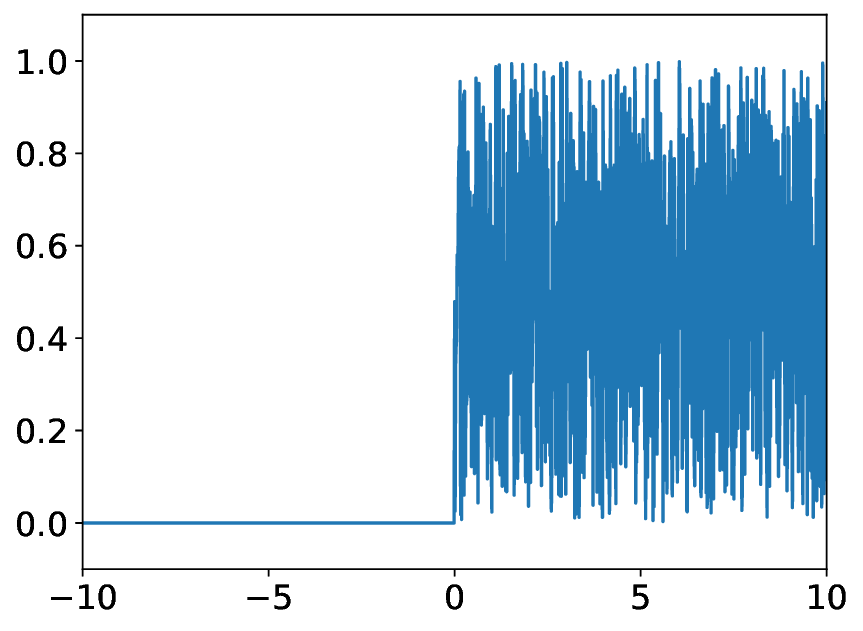}
    \label{fig:posjam}}
  \hfill
  \subfloat[Rectangular-Jamming]{\includegraphics[scale=0.30]{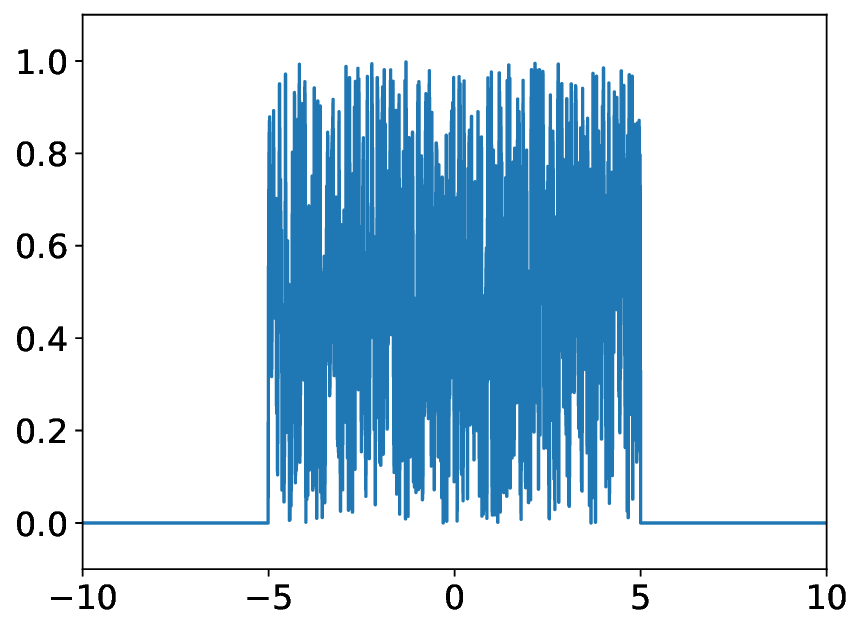}
    \label{fig:rectjam}}
  \caption{Noise functions used for gradient modulation in Experiment~6: (a) Full-Jamming, (b) Positive-Jamming, and (c) Rectangular-Jamming. In (a), \(y\) is randomized for all \(x \in [-10,10]\); in (b), \(y=0\) for \(x<0\) and randomized for \(x \ge 0\); in (c), \(y\) is randomized for \(|x| \le 5\) and \(0\) otherwise. Each function replaces activation-gradient contributions with randomized values.}
  \label{fig:jamming}
\end{figure}

The results from Experiment 5 highlight a compelling aspect of neural network training: networks exhibit considerable resilience to random perturbations in gradient magnitudes, largely maintaining their classification accuracy despite stochastic interference. Specifically, Table~\ref{exp5_1} demonstrates that CNN performance remains remarkably consistent even when the gradient amplitude is entirely randomized (Full-Jamming). A comparative analysis between the performance metrics of jammed gradient configurations (Table~\ref{exp5_1}) and their traditional gradient counterparts (referenced from earlier experiments) reveals closely aligned accuracy levels.

This robustness strongly supports the hypothesis that gradient direction, rather than its precise magnitude, serves as the primary driver guiding the network's training trajectory. The directional information provided by gradients appears sufficient to ensure effective model convergence and successful classification, relegating gradient amplitude to a secondary role.

\begin{table}[!ht]
\setlength\tabcolsep{0pt}
\caption{Experiment 5 results (Full-Jamming). Mean test accuracy (with standard deviations) for LeNet5 trained with Logistic, ReLU, and Linear activation functions using a Full-Jamming (FJ) gradient modulation approach (see Figure~\ref{fig:fulljam}). Results across various batch sizes (BS) and learning rates (LR) are reported. Boldface indicates the best performance per row.}
\label{exp5_1}
{\begin{tabular*}{\columnwidth}{@{\extracolsep\fill}@{\hskip6pt}lllll@{\hskip6pt}}
\toprule
\multicolumn{2}{l}{{Configuration}} & \textit{Log/FJ} & \textit{ReLU/FJ} & \textit{Linear/FJ}\\
\cmidrule(l{0pt}r{5pt}){1-2}
{BS} & {LR} & & &\\
\midrule
32  & 0.001 & 0.4790 (0.1134) & 0.9225 (0.0069) & 0.8969 (0.0026)\\
    & 0.002 & 0.5452 (0.2516) & 0.9373 (0.0037) & 0.9028 (0.0029)\\
    & 0.005 & 0.7846 (0.0227) & \textbf{0.9515 (0.0062)} & 0.9099 (0.0036)\\
    & 0.01  & 0.8467 (0.0166) & 0.1022 (0.0061) & 0.9127 (0.0032)\\
    & 0.02  & \textbf{0.8543 (0.0366)} & 0.1006 (0.0064) & \textbf{0.9160 (0.0026)}\\
\midrule
64  & 0.001 & 0.1611 (0.0892) & 0.8981 (0.0046) & 0.8872 (0.0027)\\
    & 0.002 & 0.4961 (0.1274) & 0.9216 (0.0025) & 0.8960 (0.0022)\\
    & 0.005 & 0.6906 (0.0789) & 0.9453 (0.0059) & 0.9041 (0.0027)\\
    & 0.01  & 0.7827 (0.0549) & 0.9425 (0.0115) & 0.9101 (0.0026)\\
    & 0.02  & 0.8260 (0.0333) & 0.1006 (0.0064) & 0.9145 (0.0033)\\
\midrule
128 & 0.001 & 0.1099 (0.0027) & 0.8598 (0.0133) & 0.8718 (0.0043)\\
    & 0.002 & 0.1195 (0.0146) & 0.9029 (0.0064) & 0.8881 (0.0031)\\
    & 0.005 & 0.5080 (0.1367) & 0.9288 (0.0049) & 0.8993 (0.0028)\\
    & 0.01  & 0.6709 (0.1254) & 0.9450 (0.0050) & 0.9055 (0.0032)\\
    & 0.02  & 0.8056 (0.0282) & 0.9311 (0.0341) & 0.9099 (0.0029)\\
\bottomrule
\end{tabular*}}
\end{table}

In the case of Full-Jamming, Table~\ref{exp5_1} highlights several critical insights regarding network behavior when the gradient magnitude is completely disrupted. Remarkably, CNNs using ReLU and Linear activations maintain high accuracy despite randomized gradient magnitudes.

Conversely, the Logistic activation exhibits significant sensitivity. This inability of the Logistic activation to converge under the Full-Jamming scenario likely arises from a combination of its bounded and symmetric output range and the complete randomization of gradient magnitudes. Specifically, Logistic outputs are strictly constrained between 0 and 1, with symmetric saturation regions near both extremes. When weight updates are forced in these regions, parameter changes become effectively random and symmetric, leading to oscillations or aimless fluctuations. Consequently, the network parameters fail to systematically approach regions of improved performance. In contrast, activations such as ReLU and Linear---having unbounded or partially unbounded output ranges---allow even randomly perturbed gradients to provide sufficient directional signals, facilitating effective convergence despite the absence of accurate gradient magnitude information.

Additionally, smaller batch sizes enhance performance, suggesting that more frequent parameter updates may partially compensate for the loss of precise gradient magnitude information. Lastly, the stability and reliability of ReLU and Linear activations remain consistently strong across various learning rates, with slightly greater consistency observed at larger batch sizes.

\begin{table}[!ht]
\setlength\tabcolsep{0pt}
\caption{Experiment 5 results (Positive-Jamming). Mean test accuracy (with standard deviations) for LeNet5 with Logistic, ReLU, and Linear activations using Positive-Jamming (PJ) gradient modulation (see Figure~\ref{fig:posjam}). Results for different batch sizes (BS) and learning rates (LR) are provided. Boldface highlights the best accuracy per row.}
\label{exp5_2}
{\begin{tabular*}{\columnwidth}{@{\extracolsep\fill}@{\hskip6pt}lllll@{\hskip6pt}}
\toprule
\multicolumn{2}{l}{{Configuration}} & \textit{Log/PJ} & \textit{ReLU/PJ} & \textit{Linear/PJ}\\
\cmidrule(l{0pt}r{5pt}){1-2}
{BS} & {LR} & & &\\
\midrule
32  & 0.001 & 0.1099 (0.0027) & 0.9211 (0.0050) & 0.8908 (0.0032)\\
    & 0.002 & 0.2052 (0.0708) & 0.9452 (0.0025) & 0.8188 (0.0967)\\
    & 0.005 & 0.2356 (0.0460) & 0.9699 (0.0025) & 0.0981 (0.0014)\\
    & 0.01  & 0.2352 (0.0404) & 0.9776 (0.0010) & 0.0981 (0.0014)\\
    & 0.02  & 0.1788 (0.0761) & \textbf{0.9831 (0.0019)} & 0.0981 (0.0014)\\
\midrule
64  & 0.001 & 0.1099 (0.0027) & 0.8865 (0.0088) & 0.8754 (0.0033)\\
    & 0.002 & 0.1099 (0.0027) & 0.9197 (0.0065) & 0.8925 (0.0037)\\
    & 0.005 & 0.2010 (0.0765) & 0.9555 (0.0038) & 0.0981 (0.0014)\\
    & 0.01  & 0.2329 (0.0394) & 0.9691 (0.0021) & 0.0981 (0.0014)\\
    & 0.02  & \textbf{0.2647 (0.0751)} & 0.9782 (0.0029) & 0.0981 (0.0014)\\
\midrule
128 & 0.001 & 0.1099 (0.0027) & 0.5203 (0.2741) & 0.8338 (0.0170)\\
    & 0.002 & 0.1099 (0.0027) & 0.8851 (0.0119) & 0.8764 (0.0039)\\
    & 0.005 & 0.1099 (0.0027) & 0.9297 (0.0051) & \textbf{0.8957 (0.0034)}\\
    & 0.01  & 0.2085 (0.0500) & 0.9501 (0.0039) & 0.0981 (0.0014)\\
    & 0.02  & 0.2296 (0.0774) & 0.9664 (0.0012) & 0.0981 (0.0014)\\
\bottomrule
\end{tabular*}}
\end{table}

\begin{table}[!ht]
\setlength\tabcolsep{0pt}
\caption{Experiment 5 results (Rectangular-Jamming). Mean test accuracy (standard deviation in parentheses) of LeNet5 with Logistic, ReLU, and Linear activation functions under Rectangular-Jamming (RJ) gradient modulation (see Figure~\ref{fig:rectjam}). Results are reported for different batch sizes (BS) and learning rates (LR). Bold values indicate the highest accuracy per row.}
\label{exp5_3}
{\begin{tabular*}{\columnwidth}{@{\extracolsep\fill}@{\hskip6pt}lllll@{\hskip6pt}}
\toprule
\multicolumn{2}{l}{{Configuration}} & \textit{Log/RJ} & \textit{ReLU/RJ} & \textit{Linear/RJ}\\
\cmidrule(l{0pt}r{5pt}){1-2}
{BS} & {LR} & & &\\
\toprule
32  & 0.001 & 0.7807 (0.0182) & 0.9121 (0.0064) & 0.8971 (0.0026)\\
    & 0.002 & 0.9227 (0.0047) & 0.9349 (0.0044) & 0.9023 (0.0024)\\
    & 0.005 & 0.9670 (0.0026) & 0.9565 (0.0018) & 0.8931 (0.0275)\\
    & 0.01  & 0.9799 (0.0011) & \textbf{0.9620 (0.0062)} & 0.8350 (0.0708)\\
    & 0.02  & \textbf{0.9845 (0.0010)} & 0.9266 (0.0169) & 0.6144 (0.4102)\\
64  & 0.001 & 0.1451 (0.0560) & 0.8934 (0.0049) & 0.8871 (0.0017)\\
\midrule
    & 0.002 & 0.7850 (0.0188) & 0.9171 (0.0042) & 0.8974 (0.0024)\\
    & 0.005 & 0.9373 (0.0027) & 0.9443 (0.0024) & 0.9047 (0.0032)\\
    & 0.01  & 0.9674 (0.0021) & 0.9558 (0.0061) & 0.9014 (0.0206)\\
    & 0.02  & 0.9800 (0.0012) & 0.9545 (0.0073) & 0.8773 (0.0709)\\
\midrule
128 & 0.001 & 0.1099 (0.0027) & 0.8473 (0.0092) & 0.8686 (0.0030)\\
    & 0.002 & 0.2332 (0.1232) & 0.8935 (0.0062) & 0.8875 (0.0023)\\
    & 0.005 & 0.8546 (0.0126) & 0.9221 (0.0043) & 0.8990 (0.0035)\\
    & 0.01  & 0.9347 (0.0040) & 0.9409 (0.0045) & 0.9051 (0.0026)\\
    & 0.02  & 0.9680 (0.0023) & 0.9548 (0.0054) & \textbf{0.9099 (0.0038)}\\
\bottomrule
\end{tabular*}}
\end{table}

The results from the Positive-Jamming and Rectangular-Jamming scenarios further reinforce the key insights derived from the Full-Jamming experiment. Specifically, \hyperref[exp5_2]{Tables~\ref{exp5_2}} and~\ref{exp5_3} illustrate that CNNs employing ReLU and Linear activations maintain high accuracy despite selective gradient magnitude disruptions, particularly at higher learning rates. ReLU activation consistently demonstrates robustness in both Positive-Jamming and Rectangular-Jamming scenarios, achieving accuracy levels up to 98~\%. Similarly, the Linear activation function exhibits strong adaptability, delivering comparable accuracy levels under these conditions.

In contrast, the Logistic activation continues to show pronounced sensitivity. Under Positive-Jamming conditions, its performance remains close to random chance, indicating persistent vulnerability to gradient magnitude perturbations in the saturation regime. This is further confirmed by the convergence of Logistic in the case of Rectangular-Jamming.

Collectively, these findings reaffirm the feasibility of breaking the traditional symmetry between forward activation functions and backward gradient computations. They demonstrate that explicitly calculating the activation function's contribution to gradients---and consequently to weight corrections---is not strictly necessary for effective neural network training. Instead, neural networks primarily rely on the directional information provided by gradients, with activations such as ReLU and Linear exhibiting substantial robustness against gradient magnitude disruptions. Conversely, Logistic activation, with its bounded and symmetric output range, shows significant vulnerability, indicating inherent limitations when subjected to gradient magnitude perturbations.

\subsection{Experiment 6: Linear Activation}

Moving from the findings from Experiment 5, this experiment aims to explore the feasibility of training a neural network (LeNet5) using a linear activation function, specifically the identity function. Specifically, we investigate breaking the forward-backward symmetry by employing a jammed gradient during backpropagation. The experiment is conducted using the Fashion-MNIST dataset.

The experiment investigates the effectiveness of gradient modulation strategies in compensating for the inherent limitations of linear activation. It is widely acknowledged in neural network literature that nonlinear activation functions are crucial for enabling networks to capture and represent complex data patterns. Conversely, linear activations, notably the identity function, restrict the expressive capability of neural networks, a limitation clearly illustrated by the performance results presented in Table~\ref{tab:fasion_basis}. These results demonstrate that networks using a purely linear activation configuration, i.e. the identity function (Linear/1), fail to achieve meaningful learning, consistently yielding low accuracy, while networks utilizing ReLU-induced nonlinearity (ReLU/H) attain significantly improved performance.

\begin{table}[]
\caption{Experiment 6 results. Mean test accuracy (with standard deviations in parentheses) of LeNet5 trained on Fashion-MNIST, keeping ht the forward-backward symmetry: Linear activation with constant gradient (\textit{Linear/1}) and ReLU activation with its Heaviside derivative (\textit{ReLU/H}). Results across various batch sizes (BS: 128, 256, 512) and learning rates (LR: 0.005, 0.01, 0.05, 0.1) demonstrate the critical role of ReLU-induced nonlinearity in achieving competitive accuracy, and underline the sensitivity of both models to hyperparameter selection. Boldface highlights the best-performing Linear activation cases.}
\label{tab:fasion_basis}
\begin{tabular*}{\textwidth}{@{\extracolsep{\fill}} cc|cc}
\toprule
\multicolumn{2}{l}{\textbf{Configuration}} & Linear/1 & ReLU/H \\
\textbf{BS} & \textbf{LR} &  &  \\ \midrule
128 & 0.005 & 0.1001 (0.0029) & 0.1460 (0.0048) \\
 & 0.01 & 0.1001 (0.0029) & 0.1991 (0.0101) \\
 & 0.05 & 0.1525 (0.0266) & 0.6833 (0.0033) \\
 & 0.1 & \textbf{0.5941 (0.0108)} & 0.7610 (0.0027) \\
\midrule
256 & 0.005 & 0.1001 (0.0029) & 0.8674 (0.0028) \\
 & 0.01 & 0.1001 (0.0029) & 0.8858 (0.0032) \\
 & 0.05 & 0.1528 (0.0247) & \textbf{0.8970 (0.0043)} \\
 & 0.1 & 0.4427 (0.0331) & 0.7716 (0.2898) \\
\midrule
512 & 0.005 & 0.1001 (0.0029) & 0.1345 (0.0044) \\
 & 0.01 & 0.1001 (0.0029) & 0.1865 (0.0051) \\
 & 0.05 & 0.1521 (0.0243) & 0.6563 (0.0030) \\
 & 0.1 & 0.3077 (0.0275) & 0.7435 (0.0026) \\ 
\bottomrule
\end{tabular*}
\end{table}

However, as discussed in Section 2, the role of any activation function can be interpreted as distributing activation values along a continuum, effectively modulating the response of parallel hyperplanes defined solely by the unit weights. These weights alone are responsible for rotating and translating the hyperplanes within the input space. While the geometric positioning—both orientation (rotation) and displacement (translation)—of these hyperplanes is entirely governed by the weights and bias terms, the activation function maps each linear output (associated with a particular hyperplane) into activation values. Therefore the contribution of activation function is given to backward path. From a theoretical perspective, it is insightful to reinterpret the role of the ReLU activation function by emphasizing two primary considerations:

\begin{enumerate}
\item \textbf{Absence of Saturation for Positive Inputs:} Unlike activation functions such as sigmoid or tanh, which experience gradient saturation for large input magnitudes, ReLU does not saturate for positive inputs. Specifically, the function remains linear for positive inputs, thereby facilitating unobstructed gradient propagation. This characteristic significantly contributes to its widespread adoption and success in deep neural network architectures.

\item \textbf{Emergent Sparsity via Zero Gradient for Negative Inputs:} For negative linear outputs, ReLU outputs zero, naturally resulting in a zero gradient during backpropagation. Although this behavior is not explicitly designed to halt weight updates, it inherently induces sparsity within the network. Neurons with negative linear outputs become inactive and consequently do not propagate gradient updates to their associated weights. This emergent sparsity allows networks to allocate computational resources selectively toward features actively contributing to output generation, enhancing both efficiency and representational effectiveness.
\end{enumerate}

\begin{table}[]
\caption{Experiment 6 results. Test accuracy (mean $\pm$ std) of LeNet5 on Fashion-MNIST with linear activation, under different noisy gradient modulation strategies: Full-Jamming (FJ), Positive-Jamming (PJ), and Rectangular-Jamming (RJ), across varying batch sizes (BS) and learning rates (LR).}

\label{tab:fasion_lin_jam}
\begin{tabular*}{\textwidth}{@{\extracolsep{\fill}} cc|ccc}
\toprule
\multicolumn{2}{l}{\textbf{Configuration}} & \textit{Linear/FJ} & \textit{Linear/PJ} & \textit{Linear/RJ} \\
\textbf{BS} & \textbf{LR} &  &  &  \\ \midrule
32 & 0.1 & 0.5074 (0.0150) & 0.3676 (0.0276) & 0.5102 (0.0172) \\
 & 0.5 & 0.7758 (0.0025) & 0.7504 (0.0027) & 0.7753 (0.0024) \\
 & 1.0 & 0.8113 (0.0027) & 0.7913 (0.0025) & 0.8106 (0.0030) \\
 & 1.2 & \textbf{0.8231 (0.0032)} & \textbf{0.7994 (0.0024)} & \textbf{0.8231 (0.0029)} \\
\midrule
64 & 0.1 & 0.4616 (0.0209) & 0.3628 (0.0288) & 0.4618 (0.0220) \\
 & 0.5 & 0.7728 (0.0023) & 0.7449 (0.0030) & 0.7721 (0.0022) \\
 & 1.0 & 0.8073 (0.0022) & 0.7893 (0.0024) & 0.8070 (0.0022) \\
 & 1.2 & 0.8194 (0.0031) & 0.7969 (0.0027) & 0.8190 (0.0026) \\
\midrule
128 & 0.1 & 0.3852 (0.0280) & 0.3574 (0.0305) & 0.3859 (0.0272) \\
 & 0.5 & 0.7691 (0.0023) & 0.7380 (0.0028) & 0.7684 (0.0019) \\
 & 1.0 & 0.8035 (0.0024) & 0.7868 (0.0023) & 0.8029 (0.0022) \\
 & 1.2 & 0.8146 (0.0030) & 0.7941 (0.0027) & 0.8146 (0.0025) \\
 \bottomrule
\end{tabular*}
\end{table}

Building upon these findings and the theoretical considerations discussed in Section 2, this experiment investigates the impact of deliberately modulating gradient updates—referred to as gradient jamming—on the learning capability of networks using linear activation functions. Specifically, three gradient modulation strategies are assessed: Full-Jamming (FJ), Positive-Jamming (PJ), and Rectangular-Jamming (RJ), with performance outcomes presented in Table~\ref{tab:fasion_lin_jam}. Results indicate that while the Full-Jamming strategy mirrors the ineffectiveness observed in purely linear configurations, restricted modulation approaches such as Positive-Jamming and Rectangular-Jamming markedly enhance the network's ability to learn, achieving accuracy levels, that although not satisfactory, are at least comparable to those observed with standard ReLU activation. This is notably interesting, even when ReLU is enhanced by Batch Normalization (BN) and Gradient Clipping (GC), as shown in Table~\ref{tab:relu_bn_gc}.

\begin{table}[!ht]
\setlength\tabcolsep{0pt}
\caption{Experiment 6 results. Mean test accuracy (with standard deviations in parentheses) of LeNet5 on the Fashion-MNIST dataset, comparing standard ReLU activation with its Heaviside derivative (\textit{ReLU/H}), ReLU activation combined with Batch Normalization (\textit{ReLU/H+BN}), and ReLU activation with Gradient Clipping (\textit{ReLU/H+GC}). Results are presented for varying learning rates (LR: 0.005, 0.01, 0.05, 0.1) with a fixed batch size (BS = 256). Boldface highlights the best-performing configuration for each learning rate.}
\label{tab:relu_bn_gc}
\begin{tabular*}{\textwidth}{@{\extracolsep{\fill}} cc|ccc}
\toprule
\multicolumn{2}{c}{\textbf{Configuration}} & ReLU/H & ReLU/H + BN & ReLU/H + GC \\
\textbf{BS} & \textbf{LR} & & & \\
\midrule
256 & 0.005 & 0.8674 (0.003) & 0.8978 (0.006) & 0.8685 (0.003) \\
    & 0.01  & 0.8858 (0.003) & 0.9039 (0.005) & 0.8814 (0.004) \\
    & 0.05  & \textbf{0.8970 (0.004)} & 0.9017 (0.004) & \textbf{0.8887 (0.005)} \\
    & 0.1   & 0.7716 (0.290) & \textbf{0.9057 (0.002)} & 0.8514 (0.178) \\
\bottomrule
\end{tabular*}
\end{table}

{
\subsection{Summary}

Our series of experiments systematically explored the feasibility and implications of breaking the conventional symmetry between forward and backward propagation in neural network training. A concise summary of each experiment, dataset, configuration details, and their corresponding key findings are provided in Table~\ref{tab:summary}.

\begin{table}[ht!]
\centering
\caption{Summary of experiments and key findings}
\label{tab:summary}
\renewcommand{\arraystretch}{1.2}
\resizebox{\textwidth}{!}{%

\begin{tabular}{@{}c>{\arraybackslash}p{0.25\textwidth}ll>{\arraybackslash}p{0.45\textwidth}@{}}
\toprule
\textbf{\#} & \textbf{Experiment} & \textbf{Dataset} & \textbf{Configurations} & \textbf{Key Findings} \\ 
\midrule
1 & \begin{tabular}[c]{@{}>{\raggedright\arraybackslash}p{0.25\textwidth}@{}}Single Unit Classifier (SUC)\end{tabular} & Ones & 
\begin{tabular}[c]{@{}l@{}}Log/d(Logistic)\\ Log/1\end{tabular} & 
\begin{tabular}[c]{@{}>{\raggedright\arraybackslash}p{0.45\textwidth}@{}}Breaking forward-backward symmetry does not influence linear separability.\end{tabular} \\ 
\midrule
2 & \begin{tabular}[c]{@{}>{\raggedright\arraybackslash}p{0.25\textwidth}@{}}Multilayer Perceptron\end{tabular} & Digits & 
\begin{tabular}[c]{@{}l@{}}Log/d(Log)\\ Log/1\end{tabular} & 
\begin{tabular}[c]{@{}>{\raggedright\arraybackslash}p{0.45\textwidth}@{}}Gradient modulation does not significantly impact standard multilayer network performance.\end{tabular} \\ 
\midrule
3 & \begin{tabular}[c]{@{}>{\raggedright\arraybackslash}p{0.25\textwidth}@{}}Convolutional Neural Networks\end{tabular} & MNIST & 
\begin{tabular}[c]{@{}l@{}}Log/d(Log)\\ Log/1\\ Log/Rect\end{tabular} & 
\begin{tabular}[c]{@{}>{\raggedright\arraybackslash}p{0.45\textwidth}@{}}Gradient modulation achieves comparable or potentially improved performance (e.g., faster convergence).\end{tabular} \\ 
\midrule
4 & \begin{tabular}[c]{@{}>{\raggedright\arraybackslash}p{0.25\textwidth}@{}}CNNs + ReLU\end{tabular} & MNIST & 
\begin{tabular}[c]{@{}l@{}}ReLU/H\\ ReLU/d(Log)\\ ReLU/1\\ ReLU/Triang\\ ReLU/Rect\end{tabular} & 
\begin{tabular}[c]{@{}>{\raggedright\arraybackslash}p{0.45\textwidth}@{}}Findings from Experiment 3 are confirmed for CNNs with ReLU as forward activation.\end{tabular} \\ \midrule
5 & \begin{tabular}[c]{@{}>{\raggedright\arraybackslash}p{0.25\textwidth}@{}}Randomized Gradient Modulation\end{tabular} & MNIST & 
\begin{tabular}[c]{@{}l@{}}Log/FJ\\ ReLU/FJ\\ Linear/FJ\\ Log/PJ\\ ReLU/PJ\\ Linear/PJ\\ Log/RJ\\ ReLU/RJ\\ Linear/RJ\end{tabular} & 
\begin{tabular}[c]{@{}>{\raggedright\arraybackslash}p{0.45\textwidth}@{}}Gradient randomization does not adversely affect learnability; linear networks can also learn effectively.\end{tabular} \\ 
\midrule
6 & \begin{tabular}[c]{@{}>{\raggedright\arraybackslash}p{0.25\textwidth}@{}}Linear Activation\end{tabular} & \begin{tabular}[c]{@{}l@{}}Fashion\\MNIST\end{tabular} & 
\begin{tabular}[c]{@{}l@{}}Linear/1\\ ReLU/H\\ ReLU/H + BN\\ ReLU/H + GC\\ Linear/FJ\\ Linear/PJ\\ Linear/RJ\end{tabular} & 
\begin{tabular}[c]{@{}>{\raggedright\arraybackslash}p{0.45\textwidth}@{}}Trainability of linear networks is confirmed, achieving results comparable to ReLU at some extent.\end{tabular} \\ 
\bottomrule
\end{tabular}%
}
\end{table}

The overall experimental findings strongly support the hypothesis that decoupling the forward activation functions from their corresponding backward gradient computations is feasible and it can offer interesting insights regarding the training of neural networks. Experiments conducted across diverse architectures, including simple Single Unit Classifiers (SUC), Multi-Layer Perceptrons (MLP), and more complex Convolutional Neural Networks (LeNet5), consistently showed that employing gradient modulations independent of the forward activation derivative does not inherently hinder network learning.

Particularly noteworthy are results from experiments involving stochastic gradient modulations (e.g., Full-Jamming, Rectangular-Jamming, and Positive-Jamming), where gradient magnitudes were entirely randomized. Despite these perturbations, the network consistently retained high classification performance, underscoring that the precise gradient magnitude derived from the activation function is largely redundant, provided the overall directional information remains intact.

Furthermore, an important insight is the confirmed trainability of neural networks employing linear activation functions. Although these networks have inherently limited expressiveness compared to their nonlinear counterparts, appropriate gradient modulation techniques, particularly Positive-Jamming and Rectangular-Jamming, significantly improve their performance. This finding challenges conventional assumptions and opens new research avenues for exploiting linear activations under controlled gradient conditions.

These findings advocate strongly for the broader exploration of forward-backward decoupling in neural network training. This opens the door to novel gradient modulation techniques that may further optimize learning stability, convergence speed, and computational efficiency, significantly expanding the design flexibility and practical applicability of neural network architectures.
}

{
\section{Application Example: Binary Neural Networks}

Training Binary Neural Networks (BNNs) poses a well-known challenge due to their use of binary activations and weights, typically taking values in $\{-1,+1\}$. While this binary constraint significantly reduces computational resources and memory requirements, it introduces notable obstacles in gradient-based optimization, primarily because standard backpropagation methods rely on differentiable activation functions.

The training of BNNs can be straightforwardly addressed by employing forward-backward decoupling, which separates forward activation functions from backward gradient computations, thus facilitating gradient approximation in the presence of non-differentiable operations.

\subsection{Preliminary Experiments on BNN Training}

Our experiments focused on a specific subclass of BNNs, specifically networks employing the Heaviside step function as their activation function. The capability to decouple functions during forward and backward propagations provides a foundation that supports the feasibility of training such networks. We evaluated binarized versions of Single-Unit Classifiers (SUC) and LeNet5 architectures.

According to our research objectives, one experiment involved training an SUC using a distinct configuration: the Heaviside step function for forward propagation and either the sigmoid's derivative or a constant function for backward propagation. The results, summarized in Table \ref{tab:bin_exp1}, demonstrate the practicality of this decoupled training method. Extensive evaluations across various network dimensions and configurations, as detailed in Table \ref{tab:bin_exp1}, confirm the robustness and consistency of our approach.

\begin{table}[!ht]
\setlength\tabcolsep{0pt} 
\caption{Single-Unit Classifier with BNN. The table reports mean test accuracies (with standard deviations) for a Binary Neural Network (BNN) trained using the Heaviside step function activation in the classification task from Experiment~1. Columns labeled \textit{Step/d(Log)} and \textit{Step/1} denote two distinct gradient approximation methods. \textbf{Dim} refers to input dimensionality, \textbf{BS} to batch size, and \textbf{LR} to learning rate. Boldface highlights the best result in each row.}
\label{tab:bin_exp1}
\begin{tabular*}{\textwidth}{@{\extracolsep{\fill}} ccc|cc}
\toprule
\multicolumn{3}{c}{\textbf{Configuration}} & \textit{Step/d(Log)} & \textit{Step/1} \\
\textbf{Dim} & \textbf{BS} & \textbf{LR} & & \\
\midrule
2 & 32 & 0.01 & 0.9999 (0.000) & \textbf{0.9999 (0.000)} \\
& & 0.1 & 0.9997 (0.000) & 0.9998 (0.000) \\
& 64 & 0.01 & 0.9999 (0.000) & 0.9999 (0.000) \\
& & 0.1 & \textbf{1.0000 (0.000)} & 0.9996 (0.000) \\
\hline
10 & 32 & 0.01 & 0.9991 (0.001) & \textbf{0.9992 (0.001)} \\
& & 0.1 & \textbf{0.9993 (0.001)} & 0.9986 (0.001) \\
& 64 & 0.01 & 0.9988 (0.001) & 0.9990 (0.001) \\
& & 0.1 & 0.9992 (0.001) & 0.9992 (0.001) \\
\hline
100 & 32 & 0.01 & 0.9855 (0.004) & 0.9894 (0.003) \\
& & 0.1 & \textbf{0.9900 (0.003)} & \textbf{0.9935 (0.002)} \\
& 64 & 0.01 & 0.9517 (0.009) & 0.9891 (0.003) \\
& & 0.1 & 0.9897 (0.003) & 0.9913 (0.002) \\
\bottomrule
\end{tabular*}
\end{table}

To further evaluate the scalability of our approach, we conducted an additional experiment employing a binarized LeNet5 architecture. The results, presented in Table~\ref{tab:bin_exp2}, reinforce our proposition that BNNs equipped with the Heaviside step activation function can be effectively trained using various gradient computation strategies. These findings illustrate the flexibility and robustness of our proposed approach, highlighting its potential applicability to diverse BNN architectures.

\begin{table}[!ht]
\setlength\tabcolsep{0pt} 
\caption{LeNet5 with BNN. This table reports mean test accuracies (with standard deviations) for a convolutional Binary Neural Network (BNN) trained with the Heaviside step activation function on the classification problem from Experiment~3. Four gradient approximation methods are compared: \textit{Step/d(Log)}, \textit{Step/1}, \textit{Step/Triang}, and \textit{Step/Rect}. Results are presented for various batch sizes (BS) and learning rates (LR). Boldface indicates the best performance in each row.}
\label{tab:bin_exp2}
{\small
\begin{tabular*}{\textwidth}{@{\extracolsep{\fill}} cc|cccc}
\toprule
\multicolumn{2}{c}{\textbf{Configuration}} & \textit{Step/d(Log)} & \textit{Step/1} & \textit{Step/Triang} & \textit{Step/Rect} \\
\textbf{BS} & \textbf{LR} & & & & \\
\midrule
256 & 0.01 & 0.8820 (0.003) & 0.7691 (0.030) & 0.8781 (0.003) & 0.9675 (0.001) \\
& 0.1 & \textbf{0.9487 (0.001)} & 0.8428 (0.014) & \textbf{0.9456 (0.002)} & \textbf{0.9817 (0.002)} \\
\hline
512 & 0.01 & 0.8436 (0.003) & 0.7298 (0.054) & 0.8411 (0.002) & 0.9514 (0.001) \\
& 0.1 & 0.9336 (0.002) & \textbf{0.8424 (0.011)} & 0.9308 (0.002) & 0.9801 (0.001) \\
\hline
1024 & 0.01 & 0.7737 (0.004) & 0.6806 (0.028) & 0.7684 (0.003) & 0.9121 (0.002) \\
& 0.1 & 0.9142 (0.001) & 0.8366 (0.012) & 0.9118 (0.001) & 0.9765 (0.002) \\
\bottomrule
\end{tabular*}}
\end{table}

\subsection{Comparison to Existing Literature}

Several approaches have been proposed in the literature to address the challenges encountered when training BNNs. It is useful to position our approach within the context of existing research. Although a thorough exploration of these approaches and their broader implications would be valuable, such an analysis is beyond the scope of this paper and would require more extensive discussion. Therefore, our current goal is primarily to illustrate how our theoretical contributions relate to and complement existing methods, rather than to provide an exhaustive comparison or detailed analysis.

For example, Bengio et al. introduced the Straight-Through Estimator (STE) \cite{bengio2013estimating}, a heuristic technique designed to approximate gradients by replacing non-differentiable operations in the backward pass with differentiable proxies. STE has demonstrated considerable practical efficacy. Our theoretical framework complements this by providing a formal justification for gradient approximation methods, reinforcing and extending the heuristic insights provided by STE.

Similarly, Courbariaux et al. developed BinaryConnect \cite{courbariaux2015binaryconnect}, an approach that retains full-precision weights during gradient updates but utilizes binarized weights during forward propagation, thus achieving improved gradient stability empirically. Our contribution enhances this method by providing a theoretical perspective that supports and deepens the conceptual foundations of BinaryConnect’s empirical approach.

Additionally, Rastegari et al. introduced XNOR-Net \cite{rastegari2016xnor}, which mitigates accuracy degradation associated with quantization by incorporating explicit scaling factors. Our theoretical framework further complements this practical innovation by offering analytical insights into the fundamental aspects of gradient computation, thereby strengthening the theoretical basis for improvements in training stability and accuracy.

Finally, Batch Normalization, proposed by Ioffe and Szegedy \cite{ioffe2015batch}, significantly stabilizes training dynamics and has proven especially beneficial for BNNs. Our theoretical approach aligns naturally with batch normalization by explicitly addressing gradient stability and computation, thus providing additional theoretical support for this widely-used stabilization method.
}

\section{Conclusions}

This study systematically explores the theoretical and practical implications of decoupling forward activations from backward gradient computations in neural network training. Through rigorous mathematical analyses and empirical validations, we uncovered several foundational insights.

Our theoretical propositions have explicitly demonstrated that the direction of gradient updates in neural network training is independent from the derivative of the activation function. This result fundamentally implies that the exact magnitude provided by activation function gradients can be replaced by alternative gradient modulation strategies without compromising the effectiveness of learning. We thoroughly validated this insight across a diverse range of architectures, activation functions, and gradient modulation techniques, underscoring the feasibility and robustness of decoupling the forward and backward phases of neural network training.

Empirically, our experiments demonstrate that simple gradient modulation techniques (such as using a constant value, rectangular functions, or stochastic noise) can yield performance levels comparable to or exceeding traditional gradient calculations derived directly from activation functions. Specifically:

\begin{itemize}
\item \textbf{Single Unit Classifier and MLP architectures} exhibited robustness and swift convergence when gradients were modulated by constant values rather than tied to specific activation derivatives.
\item \textbf{CNN architectures (LeNet5)} showed enhanced resilience and stability when employing gradient modulation techniques, such as rectangular and triangular functions, particularly when using ReLU activation functions.
\item \textbf{Gradient Jamming experiments} confirmed that even significant stochastic perturbations in gradient magnitude do not hinder learning, reinforcing the primacy of gradient direction over amplitude.
\end{itemize}

Crucially, our exploration into Binary Neural Networks (BNNs), presented in Section~5, provided compelling evidence of practical applications of our theoretical insights. BNN training, historically constrained by non-differentiable binary activations, significantly benefited from forward-backward decoupling. Employing alternative backward-phase gradient approximations, such as constant and stochastic gradients, facilitated effective training and competitive performance. Our results clearly demonstrated that gradient magnitude decoupling enabled the training of BNNs with activation functions like the Heaviside step function.

In conclusion, this work highlights a fundamental shift in understanding neural network optimization dynamics. By explicitly breaking the conventional forward-backward symmetry, we unveil a robust theoretical justification and practical framework for employing simplified or alternative gradient computations. This opens avenues for future research into novel training paradigms and architectures, potentially transforming traditional neural network optimization strategies. Future work will expand these insights further by exploring more complex architectures and investigating novel gradient modulation strategies tailored specifically to diverse application scenarios.

\section*{Acknowledgements}
The authors would like to thank Prof.\ Francisco (Paco) Herrera from the University of Granada and Dr.\ Andrés Herrera Poyatos from the University of Oxford for their insightful comments and suggestions.

\bibliographystyle{elsarticle-num} 
\bibliography{references}





\end{document}